\documentclass[twoside]{article}

\usepackage[accepted]{aistats2022_arxiv}

\usepackage{amsmath,amssymb,amsfonts,mathrsfs}
\usepackage[amsmath,thmmarks]{ntheorem}

\DeclareMathOperator*{\argmax}{argmax}
\usepackage{graphicx}
\usepackage{dsfont}
\usepackage{algorithm}
\usepackage{algorithmicx,algpseudocode}

\newtheorem{theorem}{Theorem}

\newtheorem{definition}[theorem]{Definition}
\newtheorem{lemma}[theorem]{Lemma}
\newtheorem{proposition}[theorem]{Proposition}
\newtheorem{assumption}[theorem]{Assumption}

\theoremstyle{nonumberplain}
\theoremheaderfont{\itshape}
\theorembodyfont{\normalfont}
\theoremseparator{.\,—}
\theoremsymbol{\ensuremath{\square}}
\newtheorem{proof}{Proof}

\usepackage[title]{appendix}
\usepackage{todonotes}

\usepackage[round]{natbib}

\begin{document}

%

%

\twocolumn[

\aistatstitle{Diversified Sampling for Batched Bayesian Optimization with Determinantal Point Processes}
\aistatsauthor{ Elvis Nava \And Mojm\'ir Mutn\'y \And  Andreas Krause }

\aistatsaddress{ ETH Zurich\\\texttt{elvis.nava@ai.ethz.ch}  \And  ETH Zurich\\\texttt{mojmir.mutny@inf.ethz.ch} \And ETH Zurich\\\texttt{krausea@inf.ethz.ch} } ]

\begin{abstract}
\looseness -1  In Bayesian Optimization (BO) we study black-box function optimization with noisy point evaluations and Bayesian priors. Convergence of BO can be greatly sped up by batching, where multiple evaluations of the black-box function are performed in a single round. The main difficulty in this setting is to propose at the same time diverse and informative batches of evaluation points. In this work, we introduce {\em DPP-Batch Bayesian Optimization (DPP-BBO)}, a universal framework for inducing batch diversity in sampling based BO by leveraging the repulsive properties of Determinantal Point Processes (DPP) to naturally diversify the batch sampling procedure. We illustrate this framework by formulating DPP-Thompson Sampling (DPP-TS) as a variant of the popular Thompson Sampling (TS) algorithm and introducing a Markov Chain Monte Carlo procedure to sample from it. We then prove novel Bayesian simple regret bounds for both classical batched TS as well as our counterpart DPP-TS, with the latter bound being tighter. Our real-world, as well as synthetic, experiments demonstrate improved performance of DPP-BBO over classical batching methods with Gaussian process and Cox process models. 
\end{abstract}

\begin{figure*}[h]
\vspace{-0.4cm}
\begin{tabular}{c@{\hskip-5mm}c@{\hskip-5mm}c}
  \includegraphics[width=60mm]{{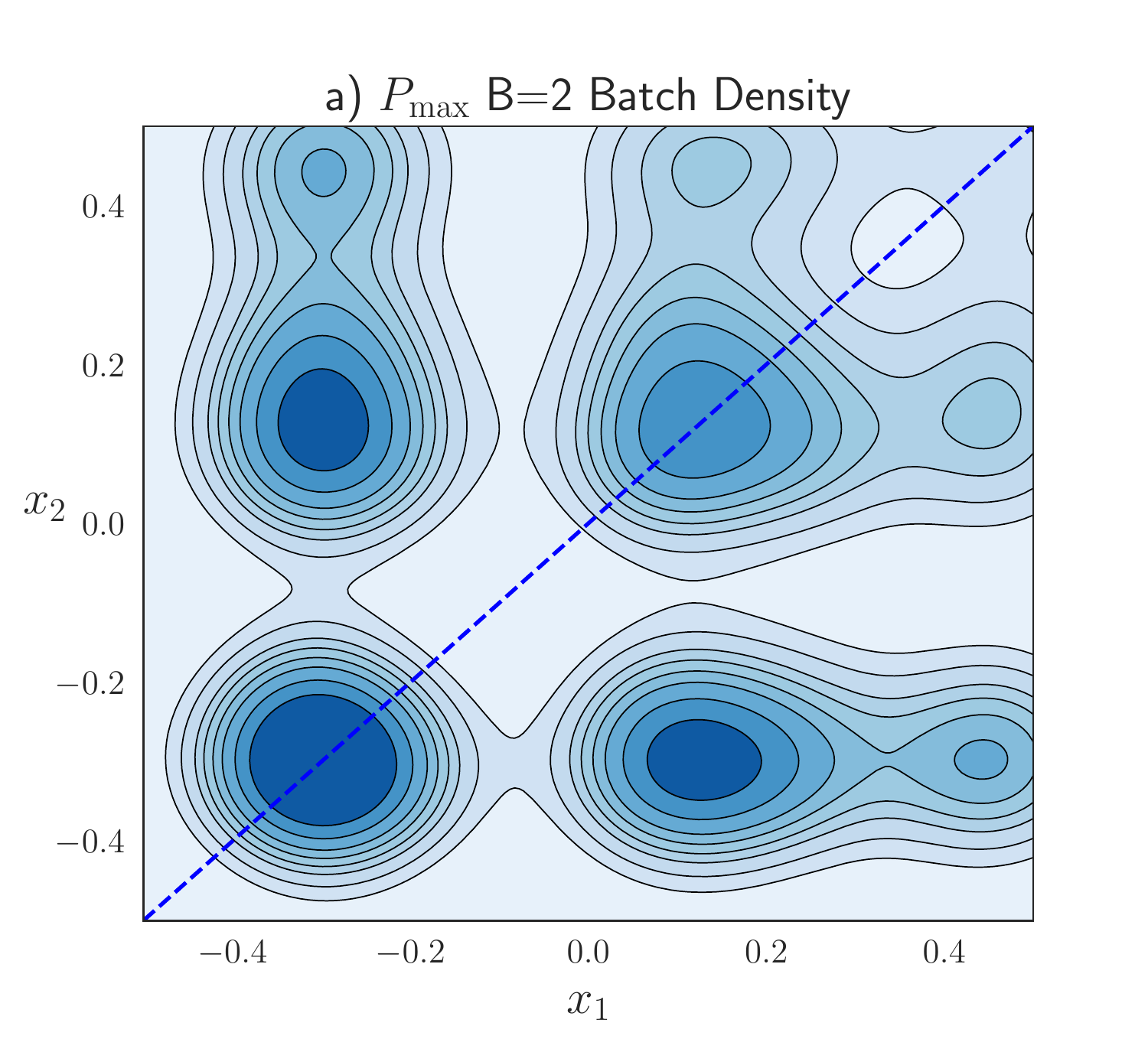}} &   \includegraphics[width=60mm]{{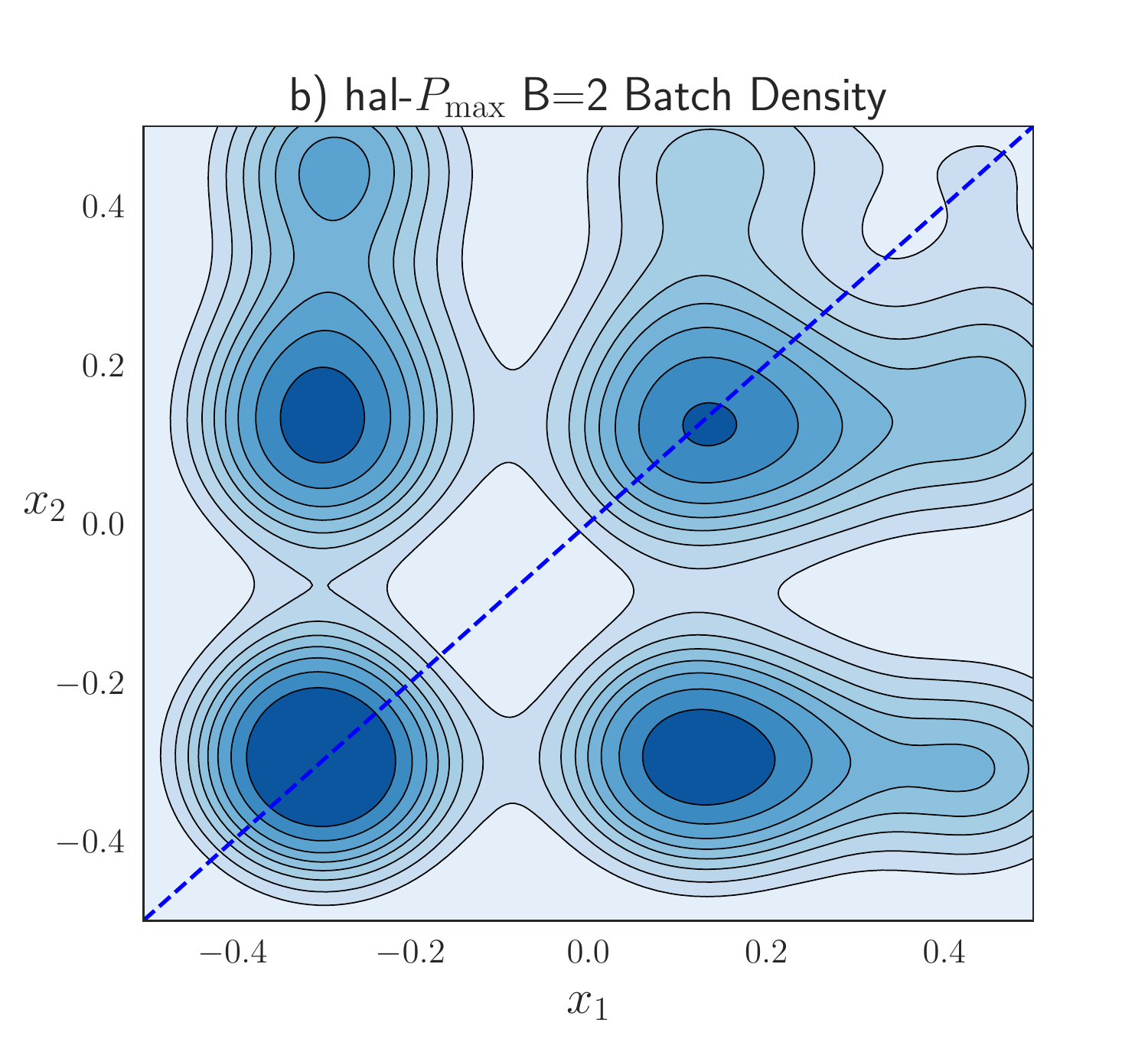}} & \includegraphics[width=60mm]{{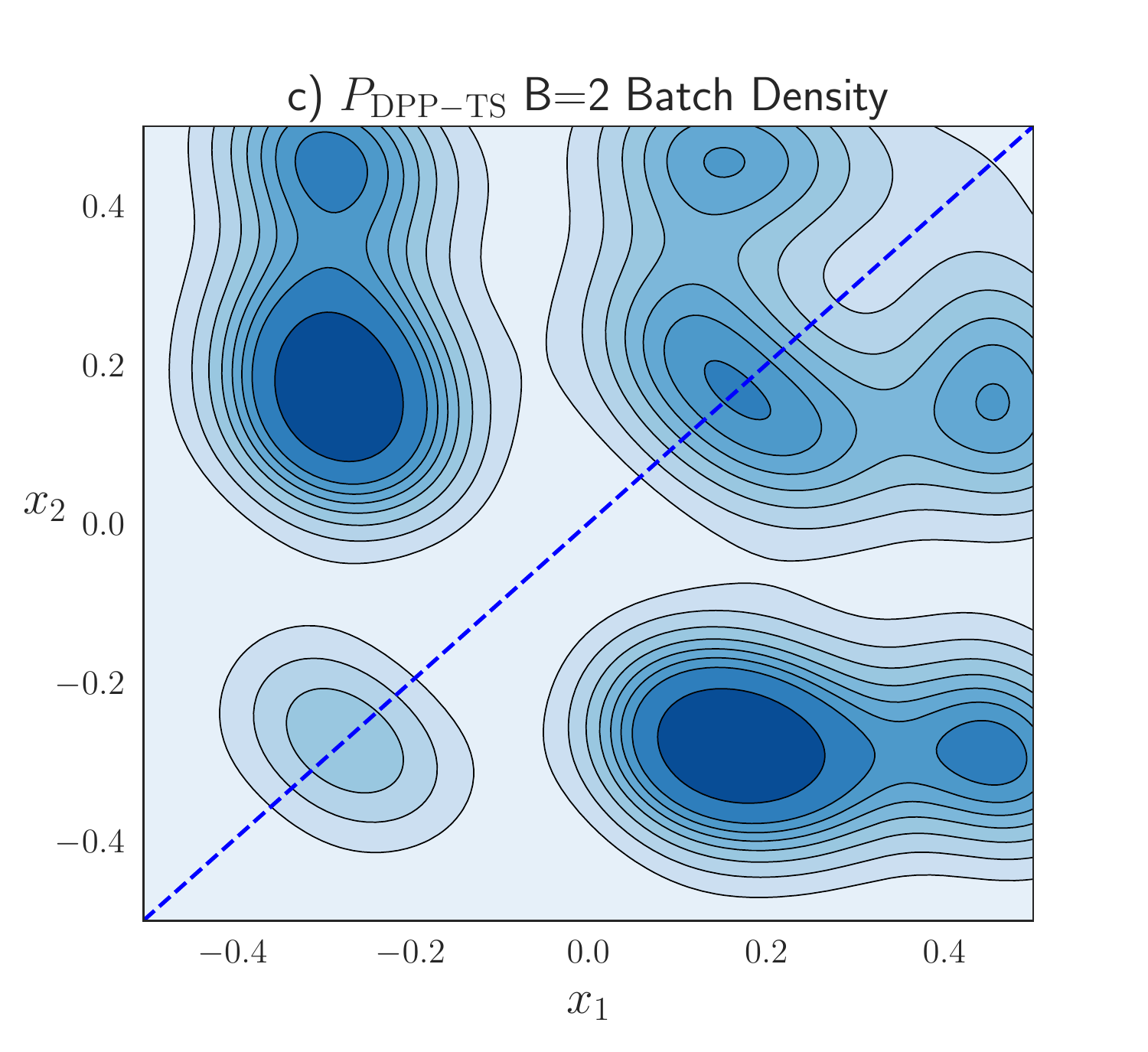}}
\end{tabular}
\vspace{-0.4cm}
\caption{\looseness -1 Diversification Demonstration. Given a Gaussian Process posterior on $f$ defined over $\mathcal{X} = [-0.5,0.5]$, we sample a batch of $B=2$ evaluation points for our next optimization iteration using a randomized Batch BO algorithm. With Thompson Sampling (a), this corresponds to sampling from the symmetric 2d distribution $P_{\max}(x_1,x_2) = p_{\max}(x_1)p_{\max}(x_2)$. We wish to sample diverse batches, therefore we would like to reduce the probability mass near the diagonal (where $x_1 = x_2$). To do so, we can use hallucinated observations (b) or our DPP-TS sampling distribution (c), which exploits DPP repulsion properties. It is apparent how $P_{\text{DPP-TS}}$, by assigning much less probability mass to locations near the diagonal, disfavors the selection of non-diverse batches.}
\label{fig:pmaxfigs}
\end{figure*}

\section{INTRODUCTION}

Gradient-free optimization of noisy black-box functions is a broadly relevant problem setting, with a multitude of applications such as de-novo molecule design \citep{gonzalez_bayesian_2015}, electron laser calibration \citep{Kirschner2019, Kirschner2019b}, and hyperparameter selection \citep{snoek_practical_2012} among many others. Several algorithms have been devised for such problems, some with theoretical guarantees, broadly referred to as Bayesian optimization \citep{Mockus1982} or multi-armed bandits \citep{berry_bandit_1985}. Our work falls into Bayesian optimization as we assume a known prior for the unknown function, and use evaluated points to update our belief about the function. 

In BO, the optimization procedure is performed sequentially by evaluating the noisy function on locations informed by past observations. In  many real world applications, multiple evaluations (experiments) can be executed in parallel. We refer to this setting as {\em batched Bayesian optimization (Batch BO)}. This is a common situation when the experimental process is easily parallelizable, such as in high-throughput wetlab experiments, or parallel training of multiple ML models on a cluster.

A main concern in the batched setting is \emph{batch diversification}: guaranteeing that the selected experimental batch does not perform redundant evaluations. We tackle this problem via Determinantal Point Processes (DPP) \citep{kulesza_determinantal_2012}, a family of repulsive stochastic processes on sets of items.
DPPs have already been successfully employed for Experimental Design \citep{derezinski_bayesian_2020}, optimization \citep{Mutny2020b}, and in combination with a deterministic batched Bayesian optimization algorithm \citep{kathuria_batched_2016}. In this work, we show how DPP-based diversification can naturally, and in a principled manner, be integrated into  randomized algorithms for BO. Of special interest is the Thompson sampling BO algorithm, which is randomized but universally applicable \citep{Thompson1933}, often empirically outperforms UCB \citep{Chapelle2011}, and in some cases has better computational properties \citep{Mutny2020}.

\subsection{Our Contribution}
In this work we introduce a framework for randomized Batched Bayesian Optimization diversification through DPPs (DPP-BBO). Our main result is an algorithm called DPP-TS, which samples from a Regularized DPP, capturing both Thompson Sampling (posterior sampling) and information-theoretic batch diversity. We use a Markov Chain Monte Carlo (MCMC) approach adapted from the DPP literature \citep{anari_monte_2016} to sample batches for this new algorithm. We establish improved Bayesian Simple Regret bounds for DPP-TS compared to classical batching schemes for Thompson Sampling, and experimentally demonstrate its effectiveness w.r.t.~BO baselines and existing techniques, both on synthetic and real-world data. 
Lastly, we demonstrate the generality of our diversification framework by applying it on an alternative randomized BO algorithm called \emph{Perturbed History Exploration} (PHE) \citep{kveton_perturbed-history_2020}; and extend it to cover Cox Process models in addition to classically assumed Gaussian Processes.

\section{BACKGROUND}\label{background}
\paragraph{Bayesian Optimization}
The problem setting for Bayesian Optimization (BO) is as follows: we select a sequence of actions $x_t \in \mathcal{X}$, where $t$ denotes the \textit{iteration count} so that $t \in [1,T]$, and $\mathcal{X}$ is the action domain, which is either discrete or continuous. For each chosen action $x_t$, we observe a noisy reward $y_t = f(x_t) + \epsilon_t$ in sequence, where $f : \mathcal{X} \rightarrow \mathbb{R}$ is the unknown reward function, and $\epsilon_t$ are assumed to be i.i.d.~Gaussian s.t.~$\epsilon_t \sim \mathcal{N}(0, \sigma^2)$ with known variance. Most BO algorithms select each point $x_t$ through  maximization of an \textit{acquisition function} $x_t = \argmax_{x \in \mathcal{X}} u_t(x|D_{t-1})$, determined by the state of an internal Bayesian model of $f$. We indicate with $D_{t-1} = \lbrace (x_1,y_1), \ldots, (x_{t-1},y_{t-1}) \rbrace$ the filtration consisting of the history of evaluation points and observations up to and including step $t-1$ on which the model is conditioned on. The main algorithmic design choices in BO are which  acquisition function and which internal Bayesian model of $f$ to use.

\paragraph{Gaussian Processes} Obtaining any  theoretical convergence guarantees in infinite or continuous domains is impossible without assumptions on the structure of $f$. A common assumption in Bayesian Optimization is that $f$ is a sample from a Gaussian Process (GP) \citep{rasmussen_gaussian_2005} prior, which has the property of being versatile yet allowing for posterior updates to be obtained in closed form. Many BO algorithms make use of an internal GP model of $f$, which is initialized as a prior and then sequentially updated from feedback. This GP is parametrized by a kernel function $k(x,x^\prime)$ and a mean function $\mu(x)$. To denote that $f$ is sampled from the GP, we write $f \sim GP(\mu, k)$.

\paragraph{Regret Minimization}
We quantify our progress towards maximizing the unknown $f$ via the notion of \textit{regret}. In particular, we define the \textit{instantaneous regret} of action $x_t$ as $r_t =f(x^\star) - f(x_t)$, with $x^\star = \argmax_{x \in \mathcal{X}} f(x)$ being the optimal action. A common objective for BO is that of minimizing {\em Bayesian Cumulative Regret} $\text{BCR}_T = \mathbb{E}\left[\sum_{t=1}^T r_t\right] = \mathbb{E}\left[\sum_{t=1}^T \left( f(x^\star) - f(x_t) \right)\right]$, where the expectation is over the prior of $f$, observation noise and algorithmic randomness. Obtaining bounds on the cumulative regret that scale sublinearly in $T$ allows us to prove convergence of the \textit{average regret} $\text{BCR}_T / T$, therefore also minimizing the {\em Bayesian Simple Regret} $\text{BSR}_T = \mathbb{E}\left[\min_{t \in [1,T]} r_t\right] = \mathbb{E}\left[\min_{t \in [1,T]} f(x^\star) - f(x_t) \right]$ and guaranteeing convergence of our optimization of $f$.

\paragraph{Batch Bayesian Optimization}
We define {\em Batch Bayesian Optimization (BBO)} as the setting where, instead of sequentially proposing and evaluating points, our algorithms propose a batch of points of size $B$ at every iteration $t$. Importantly, the batch must be finalized {\em before} obtaining any feedback for the elements within it. Batched Bayesian Optimization algorithms encounter two main challenges with respect to performance and theoretical guarantees: proposing diverse evaluation batches, and obtaining regret bounds competitive with full-feedback sequential algorithms, sublinear in the total number of experiments $BT$, where $T$ denotes the iteration count $T$ and $B$ the batch size.

\paragraph{Determinantal Point Processes (DPPs)}
\citep{kulesza_determinantal_2012} are a family of point processes characterized by the property of \textit{repulsion}. We define a point process $P$ over a set $\mathcal{X}$ as a probability measure over subsets of $\mathcal{X}$. Given a similarity measure for pairs of points in the form of a kernel function, 
Determinantal Point Processes place high probability on subsets that are \textit{diverse}  according to the kernel.

We will now describe DPPs for finite domains due to their simplicity, however their definition can be extended to continuous $\mathcal{X}$. For our purposes, we restrict our focus on L-ensemble DPPs: given a so-called L-ensemble kernel $L$ defined as a matrix over the entire (finite) domain $\mathcal{X}$, a Determinantal Point Process $P_L$ is defined as the point process such that the probability of sampling the set $X \subseteq \mathcal{X}$ is proportional to the determinant of the kernel matrix $L_X$ restricted to $X$
\begin{equation}
P_{L}(X) \propto \det\left( L_X \right)\text{.}
\end{equation}
\looseness -1 Remarkably, the required normalizing constant can be obtained in closed form as $\sum_{X \subseteq \mathcal{X}} \det(L_X) = \det(L+I)$.

If the kernel $L$ is such that $L_{ij} = l(x_i, x_j)$,  for $x_i, x_j \in \mathcal{X}$, encodes the similarity between any pair of points $x_i$ and $x_j$, then the determinant $\det\left( L_X \right)$ will be greater for diverse sets $X$. Intuitively, for the linear kernel, diversity can be measured by the area of the $|X|$-dimensional parallelepiped spanned by the vectors in $X$ \citep[see Section 2.2.1 from][]{kulesza_determinantal_2012}.

For our application, we require sampling of batches of points with a specific predetermined size. For this purpose, we focus on $k$-DPPs. A $k$-DPP $P_L^k$ over $\mathcal{X}$ is a distribution over subsets of $\mathcal{X}$ with fixed cardinality $k$, such that the probability of sampling a specific subset $X$ is proportional to that for the generic DPP case:
$P_L^k(X) = \frac{\det\left( L_X \right)}{\sum_{X^\prime \subseteq \mathcal{X}, |X^\prime| = k} \det\left(L_{X^\prime}\right)}$.

Sampling from DPPs and $k$-DPPs can be done with a number of efficient exact or approximate algorithms. The seminal exact sampling procedure for $k$-DPPs from \citet{deshpande_efficient_2010} requires time $O(kN^{\omega + 1}\log N)$ in the batch size $k$ and the size of the domain $N$, with $\omega$ being the exponent of the arithmetic complexity of matrix multiplication. This does not scale well for large domains, nor does it work for the continuous case. Fortunately, an efficient MCMC sampling scheme with complexity of $O(Nk\log(\epsilon^{-1}))$ introduced by \citet{anari_monte_2016} works much better in practice. Variants of such MCMC schemes have been proven to also work for continuous domains  \citep{rezaei_polynomial_2019}.

\section{RELATED WORK}
A number of different acquisition functions have been proposed for Bayesian Optimization, such as Probability of Improvement, Expected Improvement, Upper Confidence Bound (UCB) among many others  \citep[cf., ][]{brochu_tutorial_2010}. The Gaussian process version of UCB \citep[GP-UCB,][]{srinivas_gaussian_2010} is a popular technique based on a deterministic acquisition function, with sublinear regret bounds for common kernels.

\paragraph{Thompson Sampling}
Thompson Sampling is an intuitive and theoretically sound BO algorithm using a randomized acquisition function \citep{Thompson1933,russo_tutorial_2020}. When choosing the next evaluation point, we sample a realization from the current posterior modeling the objective function, and use this as the acquisition function to maximize $x_t = \argmax_{x \in \mathcal{X}}\tilde{f}(x)$ where $\tilde{f}$ is the sample function, e.g. $\tilde{f} \sim \text{GP}(\mu_t,K_t)$.
Bayesian Cumulative Regret was first bounded as $O(\sqrt{T\gamma_T})$ by \citet{russo_learning_2014}, where $\gamma_T$ is the maximum mutual information obtainable from $T$ observations (for more details on this well established quantity, see Appendix \ref{infobackground}). This bound matches lower bounds in $T$ \citep{Scarlett2017}. 

\paragraph{Batched UCB and Pure Exploration}\label{bucbbackground}
For Batched BO, heuristic algorithms such as Simulation Matching \citep{azimi_batch_2010} or Local Penalization \citep{pmlr-v51-gonzalez16a} attempt to solve the problem of generating informative and diverse evaluation point batches, albeit without theoretical guarantees on regret. In particular, Local Penalization selects explicitly diversified batches by greedily penalizing already-sampled points with penalization factors in the acquisition function.

\citet{desautels_parallelizing_2014} are the first to provide a theoretically justified batched algorithm, introducing GP-BUCB, a batched variant of GP-UCB. To induce diversity within batches, they use \textit{hallucinated observations}, so that ${x_{\text{GP-BUCB}}}_{t,b}$ is sampled by maximizing a UCB based on the hallucinated posterior $\tilde{D}_{t,b-1}$. The hallucinated history is constructed by using the posterior mean in place of the observed reward for points with delayed feedback.
GP-BUCB attains a cumulative regret bound of $O\left(\sqrt{TB \beta_{TB} \gamma_{TB}}\right)$, which, however, requires an \textit{initialization phase} before the deployment of the actual algorithm. For the first $T_{\text{init}}$ iterations, the evaluations are chosen with Uncertainty Sampling, picking the point satisfying $x_t = \argmax_{x \in \mathcal{X}} \sigma_t(x)$, effectively exploring the whole domain, which limits the practicality of the method. To alleviate this, \citet{contal_parallel_2013} introduce the alternative GP-UCB Pure Exploration  \citep[GP-UCB-PE,][]{contal_parallel_2013}, which mixes the UCB acquisition function with a Pure Exploration strategy. Sampling a batch at timestamp $t$, GP-UCB-PE operates in two phases: the first point of each batch is sampled with standard GP-UCB, while the remaining $B-1$ points are sampled by first defining a \textit{high probability region} $\mathfrak{R}^+$ for the maximizer, and then performing Uncertainty Sampling ${x_{\text{UCB-PE}}}_{t,b} = \argmax_{x \in \mathfrak{R}^+} \sigma_{t,b}(x)$. GP-UCB-PE's  cumulative regret is bounded by $O\left(\sqrt{TB \beta_{TB} \gamma_{TB}}\right)$ without an initialization phase, as opposed to GP-BUCB.

\paragraph{Batched TS}
\citet{kandasamy_parallelised_2018} are first to consider batching with Thompson sampling and GPs. They propose to simply resample multiple times from the posterior within each batch, effectively lifting the Thompson Sampling algorithm as-is to the batched case. By repeating TS sampling for each point within the batch, they bound the Bayesian cumulative regret by $O\left(\sqrt{TB \beta_{TB} \gamma_{TB}}\right)$. It is possible but not required to use hallucinated observations (hal-TS). However, an initialization phase identical to that of GP-BUCB is needed for the their proof on the bound to hold. A novel result from our work is an improved proof technique such that the initialization phase for Batched TS is not required for the Bayesian simple regret version of the bound to hold.

\paragraph{DPPs in Batched BO}
\citet{kathuria_batched_2016} use Determinantal Point Process sampling to define a variation of GP-UCB-PE \citep{contal_parallel_2013}, called UCB-DPP-SAMPLE.
They observe that the Uncertainty Sampling phase of GP-UCB-PE corresponds to greedy maximization of the posterior covariance matrix determinant $\det\left({K_{t,1}}_X\right)$ with respect to batches $X$ of size $B-1$ from $\mathfrak{R}^+$, with ${K_{t,1}}_X$ being the covariance matrix produced by the posterior kernel of the GP after step $(t,1)$ and restricted to the set $X$. Finding the $(B-1)$-sized submatrix of the maximum determinant is an NP-hard problem, and picking each element greedily so that it maximizes $\sigma^2_{t,b}(x) = k_{t,b}(x,x)$ fails to guarantee the best solution. Maximizing the above determinant is also equivalent to maximizing $\det\left({L_{t,1}}_X\right)$ for the DPP L-ensemble kernel defined as $L_{t,1} = I + \sigma^{-2}K_{t,1}$, called the \emph{mutual information kernel} \citep{kathuria_batched_2016}.

Instead of selecting the last $B-1$ points of each batch with Uncertainty Sampling, UCB-DPP-SAMPLE samples them from a $(B-1)$-DPP restricted to $\mathfrak{R}^+$ with the Mutual Information L-ensemble kernel $L_{t,1} = I + \sigma^{-2}K_{t,1}$. 
\citet{kathuria_batched_2016} provide a bound for UCB-DPP-SAMPLE as a variation of the $O\left(\sqrt{TB \beta_{TB} \gamma_{TB}}\right)$ bound for GP-UCB-PE. However, as we illustrate in Appendix \ref{appendix:kathuriabound}, their bound is \emph{necessarily worse} than the existing one for GP-UCB-PE.

The concurrent work of \citet{nguyen_optimal_2021} is another recent example of DPP usage in BBO diversification, proposing DPP sampling (with DPP kernel informed by a GP posterior) as a method of diverse batch selection, demonstrating good performance in experimental tasks, but no known theoretical regret guarantees.

\section{THE DPP-BBO FRAMEWORK}\label{dppbboframework}
A key insight our approach relies on is to view Thompson Sampling as a procedure that samples at each step from a \textit{maximum distribution} $p_{\max}$ over $\mathcal{X}$, so that $x_t \sim p_{\max, t}$ with
\begin{gather}
p_{\max,t}(x) = \mathbb{E}_{\tilde{f} \sim \text{Post}_t}\bigg[ \mathds{1}[x = \argmax_{x^\prime \in \mathcal{X}}\tilde{f}(x^\prime)]\bigg] \text{.}
\end{gather}

A simple approach towards Batched Thompson Sampling is to  obtain a batch $X_t$ of evaluation points (with $|X_t| = B$) by sampling $B$ times from the posterior in each round. This can again be interpreted as
\begin{equation}
X_t \sim P_{\max,t} ~ \text{with} ~ P_{\max,t}(X) = \prod_{x_b \in X} p_{\max,t}(x_b) \text{.}
\end{equation}
This way, we can view Thompson Sampling or any other randomized Batch BO algorithm as iteratively sampling from a batch distribution over $\mathcal{X}^B$ dependent on $t$. The main downside of this simple approach is that {\em independently} obtaining multiple samples may lead to redundancy. As a remedy, in our DPP-BBO framework, we modify such sampling distributions by {\em reweighing} them by a DPP likelihood.
This technique 
is general, and allows us to apply DPP diversification to {\em any }randomized BBO algorithm with batch sampling likelihood $P_{\text{A},t}(X)$.

\begin{definition}[DPP-BBO Sampling Likelihood]\label{dppbbodef}
The batch sampling likelihood of generic DPP-BBO at step $t$ is
\begin{equation}
P_{\text{DPP-BBO}, t}(X) \propto P_{\text{A}, t}(X) \det({L_t}_X)
\end{equation}
with $L_t$ being a DPP L-ensemble kernel defined over the domain $\mathcal{X}$.
\end{definition}
Notice that the domain $\mathcal{X}$ does not need to be discrete, even though we introduced the approach on discrete ground sets in order to simplify notation. This is in contrast to the existing DPP-based BO algorithm from \citet{kathuria_batched_2016}, which requires the domain to be discrete in order to efficiently sample the DPP restricted to the arbitrary region $\mathfrak{R}^+$ in the general case.

We now proceed to justify our formulation, defining the DPP-Thompson Sampling (DPP-TS) procedure in the process.

\subsection{The DPP quality-diversity decomposition}
DPPs capture element diversity but also take into account element \textit{quality} independently of the similarity measure, as illustrated by \citet{kulesza_determinantal_2012}. Namely, L-ensemble DPPs can be decomposed into a quality-diversity representation, so that the entries of the L-ensemble kernel for the DPP are expressed as $L_{ij} = q_i \phi_i^\top \phi_j q_j$ with $q_i \in \mathbb{R}^+$ representing the \textit{quality} of an item $i$, and $\phi_i \in \mathbb{R}^m$, $\|\phi_i\| = 1$ being normalized \textit{diversity} features. We also define $S$ with $S_{ij} =  \phi_i^\top \phi_j$. This allows us to represent the DPP model as $P_L(X) \propto \left(\prod_{i \in X} q_i^2 \right) \det(S_X)$.

We then consider a k-DPP with L-ensemble kernel $L$ in its quality-diversity representation, and re-weigh the quality values of items by their likelihood under a Bayesian Optimization random sampling scheme $P_{A,t}(x)$ such as Thompson Sampling $P_{\max,t}(x)$. Following this approach, we can obtain a new k-DPP likelihood by renormalizing the product of the Thompson Sampling likelihood of the batch $P_{\max}$ and an existing DPP likelihood $P_L$ for $X = \lbrace x_1, \ldots, x_B \rbrace$:
\begin{align}
\label{pbatch}
P_{\text{DPP-TS}}(X) & \propto \left(\prod_{x_b \in X} p_{\max}(x_b) \right) P_L(X)\\
& \propto \left(\prod_{x_b \in X} p_{\max}(x_b) \right) \det(L_X)\\
& \propto \left(\prod_{x_b \in X} p_{\max}(x_b) \cdot q_{x_b}^2 \right) \det(S_X)
\end{align}
The result is a k-DPP with L-ensemble kernel $\tilde{L}_{ij} = \sqrt{p_{\max}(x_i)p_{\max}(x_j)}L_{ij}$, generalizing the sampling distribution for batched TS as a stochastic process with repulsive properties. To recover original batched TS, we just need to set $L_t = I$.

\subsection{The Mutual Information Kernel}
For our choice of kernel, we follow the insight from \citet{kathuria_batched_2016} and use $L_t = I + \sigma^{-2} K_t$, with $K_t$ being the GP posterior kernel at step $t$. Consequently, the DPP loglikelihood of a set $X$ at time $t$ is proportional to the mutual information between the true function $f$ and the observations obtained from $X$: $I(f_X;\mathbf{y}_X|\mathbf{y}_{1:t-1,1:B}) = \frac{1}{2} \log \det (I + \sigma^{-2}{K_t}_X)$ (see Appendix \ref{infobackground}). This is an example of a so-called {\em Regularized k-DPP}, a k-DPP such that a symmetric positive semidefinite regularization matrix $A$ is added to an original unregularized L-ensemble DPP kernel, for the particular case of $A = \lambda I$. In such a setting, we allow for the same element to be selected multiple times and enforce that any set $X$ must have nonzero probability of being selected. By tuning the strength of the regularization, we can tune how extreme we wish our similarity repulsion to be.

\begin{definition}[DPP-TS Sampling Likelihood]\label{dpptsdef}
The batch sampling likelihood of DPP-TS at step $t$ is
\begin{equation}
P_{\text{DPP-TS}, t}(X) \propto P_{\max, t}(X) \det(I + \sigma^{-2}{K_t}_X) \text{.}
\end{equation}
\end{definition}
In Figure~\ref{fig:pmaxfigs}, we illustrate the $|X|=2$ case to compare the original $P_{\max}$ TS distribution, a TS variant with hallucinated observations, and $P_{\text{DPP-TS}}$ with its repulsion properties.

\subsection{Markov Chain Monte Carlo for DPP-BBO}

Sampling from the mutual information DPP component $\det(I + \sigma^{-2}{K_t}_X)$ of DPP-TS on its own can be done easily and efficiently, as numerous algorithms exist for both exact and approximate sampling from k-DPPs \citep{kulesza_determinantal_2012}. Likewise, we assume we are in a setting in which Thompson Sampling on its own can be performed relatively efficiently, as sampling from $P_{\max t}$ reduces to sampling a function realization $\tilde{f}$ from the posterior, e.g. $\text{GP}(\mu_t,K_t)$, and maximizing $\tilde{f}$ over $\mathcal{X}$.

However, when sampling from the product of the two distributions, we must resort to tools of approximate inference. The main issue with adopting standard approaches is that computation of the explicit likelihood $P_{\text{DPP-TS}, t}$ is {\em doubly intractable}: computation of $P_{\max, t}$ is intractable on its own, and it appears in the enumerator of $P_{\text{DPP-TS}, t}$ before normalization.

Our approach for sampling from $P_{\text{DPP-TS}, t}$ relies on a Markov Chain Monte Carlo (MCMC) sampler. We construct an ergodic Markov Chain over batches from $\Omega = \{ X ~|~ \; X \subset \mathcal{X}, |X| = k \}$ with transition kernel $T(X^{\prime} | X)$ such that the detailed balance equation $Q(X)T(X^{\prime} | X) = Q(X^{\prime})T(X | X^{\prime})$ is satisfied almost surely with respect to $P_{\text{DPP-TS}, t}$, with $Q(X) =\\ P_{\max, t}(X)\det({L_t}_X)$ being the unnormalized potential of $P_{\text{DPP-TS}, t}$.

If $Q(X)$ were tractable, we could use the standard Metropolis-Hastings algorithm \citep{hastings_monte_1970}, which satisfies the detailed balance equation. The problem with naively using Metropolis-Hasting MCMC sampling is that our $Q(X) = P_{\max, t}(X)\det({L_t}_X)$ contains $P_{\max, t}(X) = \prod_{x_b \in X} p_{\max, t}(x_b)$, which is intractable and cannot be computed on the fly. As previously stated, the only thing we can easily do is sample from it by sampling $\tilde{f}$ and then maximizing it. However, if we modify the standard Metropolis-Hastings MCMC algorithm by using $p_{\max, t}$ proposals, we obtain Algorithm \ref{algomcmc}, which satisfies detailed balance. We refer to Appendix \ref{mcmcappendix} for the proof. This algorithm can be interpreted as a variant of an existing k-DPP sampler proposed by \citet{anari_monte_2016}.
\begin{algorithm}[H]
\caption{DPP-TS MCMC sampler}\label{algomcmc}
\begin{algorithmic}
\State pick random initial batch $X$
\Repeat
	\State uniformly pick point $x_b \in X$ to replace
	\State sample candidate point $x_b^{\prime} \sim p_{\max, t}(x_b^{\prime})$
	\State define $X^{\prime} = \left(X \setminus \{x_b\}\right) \cup \{x_b^{\prime}\}$
	\State accept with probability $\alpha = \min \left\lbrace 1, \frac{\det({L_t}_{X^{\prime}})}{\det({L_t}_X)} \right
\rbrace$
	\If {accepted}
		\State $X$ = $X^{\prime}$
	\EndIf
\Until converged
\end{algorithmic}
\end{algorithm}
\begin{algorithm}[H]
\caption{DPP-TS Algorithm}\label{dpptsalgo}
\begin{algorithmic}
\item \textbf{Input:} Action space $\mathcal{X}$, GP prior $\mu_1$, $k_1(.,.)$, history $D_0 = \lbrace\rbrace$
\For{$t = 1, \ldots, T$}
	\State Sample $X_t \sim P_{\text{DPP-TS}, t}(X_t)$ with Alg.\ \ref{algomcmc}
	\State Observe $y_{t,b} = f(x_{t,b}) + \epsilon_{t,b}$ for $b \in [1,B]$
	\State Add observations to history $D_t = D_{t-1} \cup \lbrace (x_{t,1},y_{t,1}),\ldots,(x_{t,B},y_{t,B}) \rbrace$
	\State Update the GP with $D_t$ to get $\mu_{t+1}$, $k_{t+1}(.,.)$
\EndFor
\end{algorithmic}
\end{algorithm}
\subsection{DPP-TS}
Given the sampling distribution (Definition \ref{dpptsdef}) and the above MCMC algorithm, we can now fully specify the overall procedure for our DPP-TS sampling algorithm summarized in Algorithm \ref{dpptsalgo}.

\section{BAYESIAN REGRET BOUNDS}

We now establish bounds on the Bayesian regret. Instead of assuming the existence of a fixed true $f$, 
we assume that the true function is sampled from a Gaussian Process prior $f \sim GP(0,K)$.  

In particular, our regret bounds are obtained on a variant of Bayes regret called Bayes Batch Cumulative Regret $\text{BBCR}_{T,B} = \mathbb{E}\left[ \sum_{t=1}^T \min_{b \in [1,B]} r_{t,b} \right] = \mathbb{E}\left[ \sum_{t=1}^T \min_{b \in [1,B]} \left( f(x^{\star}) - f(x_{t,b}) \right) \right]$ which only considers the best instantaneous regret within each batch, as we make use of proof techniques from \citet{contal_parallel_2013} involving such a formulation. It is straightforward to see that by bounding BBCR we at the same time bound the Bayes Simple Regret (introduced in Section \ref{background}), as $\text{BSR}_{T,B} \leq \text{BBCR}_{T,B} / T$, similarly to how $\text{BSR}_{T,B} \leq \text{BCR}_{T,B} / TB$.

\subsection{Improved bound on BBCR for Batched Thompson Sampling}

Our first theoretical contribution is an improved version of the bound on Bayesian Simple Regret from \citet{kandasamy_parallelised_2018}. Our version of the algorithm requires no initialization procedure to guarantee sublinear regret in contrast to prior work.

Unlike the original Gaussian TS Bayesian bounds from \citet{russo_learning_2014}, \citet{kandasamy_parallelised_2018} analyze the problem over a continuous domain. Therefore, it requires an additional assumption previously used in the Bayesian continuous-domain GP-UCB bound \citep{srinivas_gaussian_2010}.
\begin{assumption}[Gradients of GP Sample Paths]\label{gradassumpt}
Let $\mathcal{X} \subseteq [0,l]^d$ compact and convex with $d \in \mathbb{N}$ and $l>0$, $f \sim \text{GP}(0,K)$ where $k$ is a stationary kernel. Moreover, there exist constants $a,b > 0$ such that
$P\left( \sup_{x\in \mathcal{X}} \left| \frac{\partial f(x)}{\partial x_i} \right| > L \right) \leq a e^{-(L/b)^2} \quad \forall L > 0, \forall i \in \lbrace 1, \ldots, d \rbrace$.
\end{assumption}
Using the above assumption, we can show the following theorem. 
\begin{theorem}[BBCR Bound for Batched TS]\label{tsbatchbound}
If $f \sim GP(0,K)$ with covariance kernel bounded by 1 and noise model $\mathcal{N}(0,\sigma^2)$, and either
\begin{itemize}
\item Case 1: finite $\mathcal{X}$ and $\beta_t = 2\ln\left(\frac{B(t^2 + 1)|\mathcal{X}|}{\sqrt{2\pi}}\right)$;
\item Case 2: compact and convex $\mathcal{X} \subseteq [0,l]^d$, with Assumption \ref{gradassumpt} satisfied and $\beta_t = 4(d+1)\log(Bt) + 2d\log(dab\sqrt{\pi})$.
\end{itemize}
Then Batched Thompson Sampling attains Bayes Batch Cumulative Regret of
\emph{
\begin{equation}
{\text{BBCR}_{\text{TS}}}_{T,B} \leq \frac{C_1}{B} + \sqrt{ C_2 \frac{T}{B} \beta_T \gamma_{TB}}
\end{equation}
}
with $C_1 = 1$ for Case 1, $C_1 = \frac{\pi^2}{6} + \frac{\sqrt{2\pi}}{12}$ for Case 2, and $C_2 = \frac{2}{\log(1 + \sigma^{-2})}$.
\end{theorem}
Therefore, ${\text{BSR}_{\text{TS}}}_{T,B} \leq \frac{C_1}{TB} + \sqrt{ C_2 \frac{1}{TB} \beta_T \gamma_{TB}}$. We point to Appendix \ref{tsbbcrboundproof} for proof.
The bound from \citet{kandasamy_parallelised_2018} (without an initialization phase) is similar, except for the presence of an $\exp(C)$ factor in the square root term, which scales linearly with $B$. Our version of the bound does not contain $\exp(C)$, allowing thus for sublinear regret in $B$.

\subsection{BBCR bound for DPP-TS}
We now shift the focus to our novel DPP-TS algorithm and obtain an equivalent bound. To do so, we modify the algorithm we developed and introduce DPP-TS-alt, so that for every batch:

a) For the first sample in the batch $x_{\text{DPP-TS-alt}\; t,1}$, we sample from $p_{\max, t}$ as in standard Thompson Sampling; b) For all the other samples $x_{\text{DPP-TS-alt}\; t,b}$ with $b \in [2,B]$, we sample from joint $P_{\text{DPP-TS }t}$, using the most updated posterior variance matrix $K_{t,1}$ to define the DPP kernel.

The reason why we introduced DPP-TS as such and not DPP-TS-alt in the first place is both for simplicity and because in practice their performance is virtually identical (see Appendix \ref{appendix:addexp}).
We have the following
\begin{theorem}[BBCR Bound for DPP-TS]\label{dpptsbatchbound}
Consider the same assumptions as for Theorem \ref{tsbatchbound}.
Then DPP-TS (in its DPP-TS-alt variant) attains Bayes Batch Cumulative Regret of
\emph{
\begin{equation}
{\text{BBCR}_{\text{DPP-TS}}}_{T,B} \leq \frac{C_1}{B} + \sqrt{ C_2 \frac{T}{B} \beta_T \gamma_{TB}} -C_3
\end{equation}
}
with $C_1 = 1$ for Case 1, $C_1 = \frac{\pi^2}{6} + \frac{\sqrt{2\pi}}{12}$ for Case 2, $C_2 = \frac{2}{\log(1 + \sigma^{-2})}$, and $-C_3 < 0$ (defined in Appendix \ref{dpptsbbcrboundproof}).
\end{theorem}
We can thus obtain ${\text{BSR}_{\text{DPP-TS}}}_{T,B} \leq \frac{C_1}{TB} - \frac{C_3}{T} + \sqrt{ C_2 \frac{1}{TB} \beta_T \gamma_{TB}}$. Moreover, this bound is necessarily tighter than that for standard TS: $\frac{C_1}{TB} - \frac{C_3}{T} + \sqrt{ C_2 \frac{1}{TB} \beta_T \gamma_{TB}} \leq \frac{C_1}{TB} + \sqrt{ C_2 \frac{1}{TB} \beta_T \gamma_{TB}}$. We point to Appendix \ref{dpptsbbcrboundproof} for the proof.

\begin{figure*}[h]
    \includegraphics[width=\textwidth]{{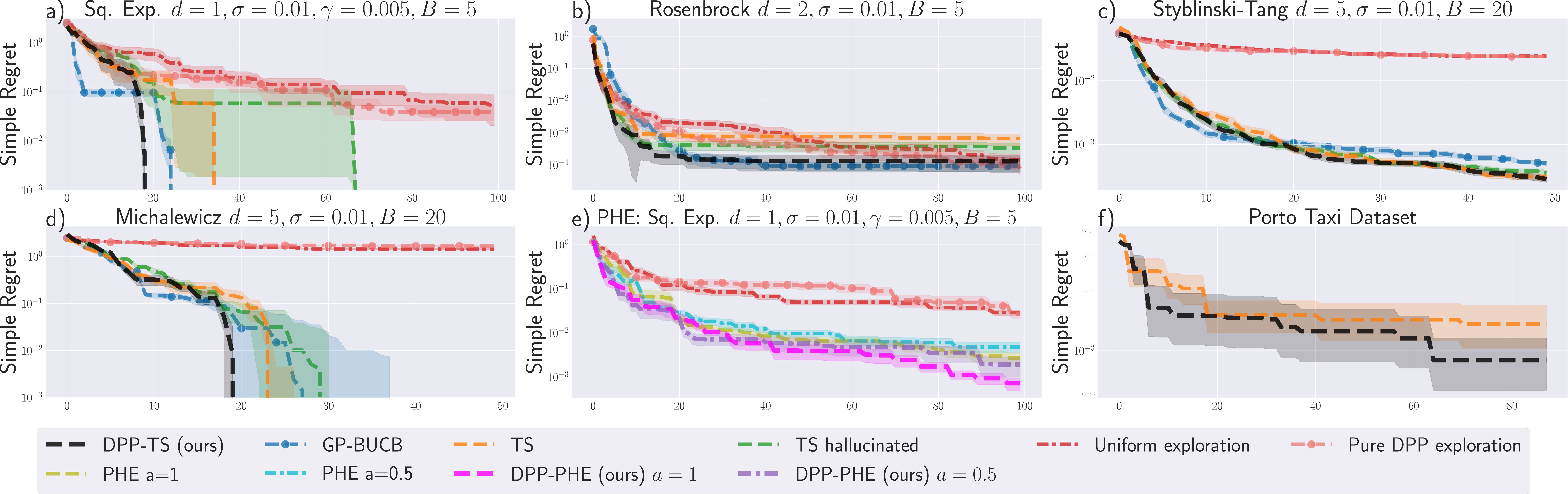}}
    \caption{Comprehensive experimental comparisons between DPP-TS and classic BBO techniques for Simple Regret (log scale): \textbf{a)} $f$ sampled from Squared Exponential GP; \textbf{b)} Rosenbrock; \textbf{c)} Styblinski-Tang; \textbf{d)} Michalewicz; \textbf{e)} PHE experiment with $f$ sampled from QFF Squared Exponential GP; \textbf{f)} Cox process sensing experiment on the Porto taxi dataset. The named functions are defined in Section \ref{dppexpsynth}. Overall, DPP-TS outperforms or equals the other algorithms, quickly sampling good maximizers thanks to improved batch diversification.}
    \vspace{-0.3cm}
    \label{fig:experimentcomp}
\end{figure*}

\section{EXPERIMENTS AND COMPARISONS}\label{expsection}


To make the case for our algorithmic framework's effectiveness in practice, we perform a series of benchmark tests on synthetic and real world optimization problems, comparing DPP-BBO against classic BBO algorithms on Simple Regret metrics. (Cumulative Regret comparisons feature in Appendix \ref{appendix:addexp}.)

\subsection{DPP-TS Comparisons on Synthetic Data} \label{dppexpsynth}

We first compare DPP-TS on synthetic benchmarks against regular batched TS, GP-BUCB, hallucinated TS (Batched Thompson Sampling with hallucinations as in GP-BUCB), Pure DPP Exploration (DPP sampling from the DPP component of DPP-TS) and Uniform Exploration (uniform random sampling over the domain). We exclude 
algorithms 
that are not applicable to continuous domains.

Figure \ref{fig:experimentcomp} details a number of such comparisons on synthetic benchmark functions under different settings, averaged over 15 experimental runs. For \ref{fig:experimentcomp}.a and \ref{fig:experimentcomp}.b we optimize over a discrete finite domain $\mathcal{X}$, using an exact Gaussian Process prior with a squared exponential kernel. The acquisition function is maximized by calculation of the explicit maximum over the discretized domain.

For \ref{fig:experimentcomp}.c and \ref{fig:experimentcomp}.d, we optimize over a continuous domain $\mathcal{X} = [0,l]^d$, using an approximate Gaussian Process prior specified with Quadrature Fourier Features \citep{mutny_efficient_2018}. These functions are additive and, hence, the optimization can be done dimension-wise. When optimizing the one-dimensional projection of the acquisition function we use first order gradient descent with restarts.

Specific benchmarks we use are the Rosenbrock function $f(x) = 100(x_2 - x_1^2)^2 + (x_1 - 1)^2$; the Stiblinski-Tang function $f(x) = \frac{1}{2} \sum_{i=i}^d \left(x_i^4 - 16x_i^2 + 5x_i \right)$; and the Michalewicz function $f(x) = -\sum_{i=i}^d \sin\left(x_i\right)\sin^{2d}\left(ix_i^2 / \pi\right)$.

Overall, DPP-TS converges very quickly to sampling good maximizers, almost always beating or at least equaling the Simple Regret performance of the other algorithms, while exhibiting low-variance behavior. The added diversity from the DPP sampling procedure appears to favor quickly finding better maxima while not getting stuck in suboptimal but high-confidence regions, as seems to often happen to GP-UCB. A series of additional experiments is discussed in Appendix \ref{appendix:addexp}, including experiments on Cumulative Regret, DPP-TS with parametrized DPP kernels, and a comparison between DPP-TS and DPP-TS-alt which shows them to be of equivalent performance in practice.

\subsection{DPP-Perturbed History Exploration}

To further demonstrate the effectiveness and versatility of the DPP-BBO framework, we apply it to the recently introduced Perturbed History Exploration (PHE) algorithm \citep{kveton_perturbed-history_2020}. PHE is a BO algorithm which is agnostic of the specific model $f_\theta$ chosen for modeling $f$. Assuming that rewards are bounded, and given a parameter $a$, the algorithm introduces {\em pseudo-rewards} $a$ for each observation in its global history, and at each step maximizes its learned perturbed $f_\theta$ to propose a new evaluation point. We can interpret this procedure as sampling from $p_{\text{PHE}, t}(x)$, with the stochastic component stemming from the pseudo-reward generation. Given this, we can define DPP-PHE as $P_{\text{DPP-PHE}, t}(X) \propto \left(\prod_{x \in X} p_{\text{PHE}, t}(x)\right) \det(I + \sigma^{-2}{K_t}_X)$ where $K_t$ is an approximation of the Bayesian posterior covariance for the $f_\theta$ model.

Figure \ref{fig:experimentcomp}.e experimentally compares PHE and DPP-PHE for $a=0.5$ and $a=1$ on a synthetic function (over a continuous $\mathcal{X}$) sampled from a 1-d squared exponential GP prior, while using as internal model a QFF GP regression. We can see that DPP-PHE improves on the Simple Regret when compared to regular PHE for the same $a$.

\subsection{DPP-TS and Cox Process Sensing}

\looseness -1 To benchmark our DPP-TS algorithm on a real world setting and demonstrate the versatility of the modeling choice, we turn to a Cox Process Sensing problem in the form of taxi routing on a 2-dimensional city grid, as considered by \citet{Mutny2021a}. Given a dataset of geo-localized taxi cab hails in Porto and a subdivision of the city into an 8x8 grid, we aim to learn the best locations where to schedule a fleet of taxis while, at beginning of each day - corresponding to a single iteration, we only observe the taxi hailing events in the grid cells which had vehicles scheduled to them. 

\looseness -1 We put a Gaussian process prior on the unknown rate function of a Poisson process, yielding a Cox Process with Poisson Process likelihood. The likelihood of observing a realization $\mathcal{D} = \lbrace x_n \rbrace_{n=1}^N$ over the domain $\mathcal{X}$ for a Poisson Process with rate function $\lambda(.)$ is $p(\mathcal{D} \mathop{} | \mathop{} \lambda(.)) = \exp(-\int_{\mathcal{X}}{\lambda(x)\mathop{}\!\mathrm{d}x})\prod_{n}{\lambda(x_n)}$. This Poisson process specification is used in the construction of a Cox process model, which is $p(\mathcal{D},\lambda(.),\Theta) = p\left(\mathcal{D}\mathop{}|\mathop{}\lambda(.)\right) \cdot p(\lambda(.)\mathop{}|\mathop{}\Theta) \cdot p(\Theta)$, with $\lambda(.)$ being a Gaussian Process conditioned on being positive-valued over the domain. We adopt the inference scheme along with the approximation scheme to maintain positivity of the rate function from \citet{Mutny2021}. The samples from the posterior are obtained via Langevin dynamics.

In our experiment, we compare TS for Cox Process Sensing from \citet{Mutny2021a} with our DPP-TS approach, leveraging our diversifying process to improve city coverage by our scheduled taxi fleets. As DPP kernel, we use the mutual information kernel that is obtained when the posterior for the rate function is approximated with a Gaussian distribution, known as the Laplace Approximation.

In Figure \ref{fig:experimentcomp}.f we depict allocation of 5 taxis to city blocks and report the simple regret. DPP-TS reliably achieves lower simple regret than standard Thompson Sampling sensing with resampling.

\section{CONCLUSIONS}
In this work we introduced DPP-BBO, a natural and easily applicable framework for enhancing batch diversity in BBO algorithms which works in more settings than previous diversification strategies: it is directly applicable to the continuous domain case, when due to approximation and non-standard models we are unable to compute hallucinations or confidence intervals (as in the Cox process example), or more generally when used in combination with any randomized BBO sampling scheme or arbitrary diversity kernel. Moreover, for DPP-TS we show improved theoretical guarantees and strong practical performance on simple regret.

\subsubsection*{Acknowledgements}
This research was supported by the ETH AI Center and
the SNSF grant 407540 167212 through the NRP 75
Big Data program. This publication was created as part of NCCR Catalysis (grant number 180544), a National Centre of Competence in Research funded by the Swiss National Science Foundation.

\bibliographystyle{apalike}
\bibliography{references}

%
%

\onecolumn \makesupplementtitle

\begin{appendices}
\section{MARKOV CHAIN MONTE CARLO SAMPLING FOR DPP-BBO}\label{mcmcappendix}
Our approach for sampling from $P_{\text{DPP-TS } t}$ leverages a Markov Chain Monte Carlo (MCMC) sampler. We construct an ergodic Markov Chain over batches from $\Omega = \{ X | \; X \subset \mathcal{X}, |X| = k \}$ with transition kernel $T(X^{\prime} | X)$ such that the Detailed Balance equation
\begin{equation}\label{detbalance}
Q(X)T(X^{\prime} | X) = Q(X^{\prime})T(X | X^{\prime})
\end{equation}
is satisfied almost surely with respect to $P_{\text{DPP-TS } t}$, with $Q(X) = P_{\max t}(X)\det(L_X)$ being the unnormalized potential of $P_{\text{DPP-TS } t}$.

Running the Markov chain will at the limit produce a limiting distribution $\pi(X)$ independent of the initial distribution $\pi_0(X)$. If the above mentioned property of Detailed Balance is satisfied, $\pi(X)$ will be equivalent to the true distribution $P_{\text{DPP-TS } t}$, meaning we can use the Markov Chain to approximately sample from $P_{\text{DPP-TS } t}$ provided we run it long enough.

If $Q(X)$ were tractable, we could use the standard Metropolis-Hastings algorithm \citep{hastings_monte_1970}: at each step sampling a candidate batch $X^{\prime}$ from a proposal distribution $R(X^{\prime}|X)$, then accepting the candidate with probability $\alpha = \min \left\lbrace 1, \frac{Q(X^{\prime})R(X|X^{\prime})}{Q(X)R(X^{\prime}|X)} \right\rbrace$.
\begin{algorithm}[H]\label{metropolisalgo}
\caption{Metropolis-Hastings MCMC}
\begin{algorithmic}
\State sample initial $X$ at random
\Repeat
	\State sample candidate $X^{\prime} \sim R(X^{\prime}|X)$
	\State accept with probability $\alpha = \min \left\lbrace 1, \frac{Q(X^{\prime})R(X|X^{\prime})}{Q(X)R(X^{\prime}|X)} \right\rbrace$
	\If {accepted}
		\State $X$ = $X^{\prime}$
	\EndIf
\Until converged
\end{algorithmic}
\end{algorithm}

\begin{theorem}[Metropolis-Hastings \citep{hastings_monte_1970}]\label{metropolistheorem}
The Markov Chain obtained from the Metropolis-Hastings Algorithm satisfies the Detailed Balance equation $Q(X)T(X^{\prime} | X) = Q(X^{\prime})T(X | X^{\prime})$ over the support of the proposal distribution $R(X^{\prime}|X)$.
\end{theorem}

\begin{proof}
We assume that $R(X^{\prime}|X) > 0 \; \forall X,X^\prime$, and we analyze the two cases for the Detailed Balance equation.
\begin{itemize}
\item Case $X=X^\prime$: The equivalence is trivial for any $T(X|X)$.
\item Case $X \neq X^\prime$:

We can express the transition kernel as $T(X^{\prime} | X) = \alpha R(X^{\prime}|X)$. Assume that for the transition $X \rightarrow X^\prime$ we have $\frac{Q(X^{\prime})R(X|X^{\prime})}{Q(X)R(X^{\prime}|X)} < 1$, and therefore $T(X^{\prime} | X) = \frac{Q(X^{\prime})R(X|X^{\prime})}{Q(X)R(X^{\prime}|X)} R(X^{\prime}|X) = \frac{Q(X^{\prime})R(X|X^{\prime})}{Q(X)}$.
Then, for the inverse transition $X^\prime \rightarrow X$ we necessarily have $\frac{Q(X^{\prime})R(X|X^{\prime})}{Q(X)R(X^{\prime}|X)} \geq 1$ and $T(X | X^\prime) = R(X|X^{\prime})$.

The resulting Detailed Balance equation is
\begin{equation}
Q(X)\frac{Q(X^{\prime})R(X|X^{\prime})}{Q(X)} = Q(X^\prime)R(X|X^{\prime})
\end{equation}
and we have equality.
\end{itemize}
\end{proof}

If the proposal distribution has the same support of the true distribution, Metropolis-Hastings allows us to approximately sample from it.

\subsection{Metropolis-Hastings with $p_{\max}$ proposals}

The problem with naively using Metropolis-Hasting MCMC sampling is that our $Q(X) = P_{\max t}(X)\det(L_X)$ contains $P_{\max}(X) = \prod_{x_b \in X} p_{\max}(x_b)$, which is intractable and cannot be computed on the fly. As previously stated, the only thing we can easily do is sample from it by sampling $\tilde{f}$ and then maximizing it. In order to obtain a suitable MCMC sampler, we need to subtly alter existing samplers.

We first propose an MCMC algorithm which samples whole batches at every step:
\begin{algorithm}[H]
\caption{Full batch MCMC sampler}
\begin{algorithmic}
\State pick random initial batch $X$
\Repeat
	\State sample candidate batch $X^{\prime} \sim P_{\max}(X^{\prime})$
	\State accept with probability $\alpha = \min \left\lbrace 1, \frac{\det(L_{X^{\prime}})}{\det(L_X)} \right
\rbrace$
	\If {accepted}
		\State $X$ = $X^{\prime}$
	\EndIf
\Until converged
\end{algorithmic}
\end{algorithm}

This algorithm is equivalent to Metropolis-Hastings: if in MH we chose $R(X^{\prime}|X) = P_{\max t}(X^{\prime})$, the fraction in the definition of the acceptance probability $\alpha$ would in fact reduce to
\begin{equation}
\frac{Q(X^{\prime})R(X|X^{\prime})}{Q(X)R(X^{\prime}|X)} = \frac{P_{\max}(X^{\prime})\det(L_{X^{\prime}})P_{\max}(X)}{P_{\max}(X)\det(L_X)P_{\max}(X^{\prime})} = \frac{\det(L_{X^{\prime}})}{\det(L_X)} \text{.}
\end{equation}
By virtue of this equivalence, Theorem \ref{metropolistheorem} applies to our procedure as well and our sampler approximately samples from the true $P_{\text{DPP-TS}}$ distribution.

In a similar fashion, it's possible to define a more efficient MCMC sampler which only changes one point from the batch at every step. Since we're in the k-DPP setting, it's possible for us to consider the distribution over a batch $X$ as a k-dimensional multivariate distribution over the $x_b \in X$. Then, the obtained sampler can be seen as akin to a Gibbs sampler, and is the one showed in the main paper as Algorithm \ref{algomcmc}.

This again reduces to Metropolis-Hastings, with proposal
\begin{equation}
R(X^{\prime}|X) =
\begin{cases}
0 & \text{if } \exists i,j : x_i^{\prime} \neq x_i \wedge x_j^{\prime} \neq x_j \\
\frac{1}{k}p_{\max}(x_i^{\prime}) & \text{if } !\exists i : x_i^{\prime} \neq x_i\\
\sum_{x_i^{\prime} \in X^\prime}{\frac{1}{k}p_{\max}(x_i^{\prime})} & \text{if } X = X^\prime
\end{cases}
\end{equation}

When sampling a proposal point with $X \neq X^\prime$, the fraction in the definition of the acceptance probability $\alpha$ then becomes
\begin{align}
\frac{Q(X^{\prime})R(X|X^{\prime})}{Q(X)R(X^{\prime}|X)} & =
\frac{\left(\prod_{x_j^{\prime} \in X^{\prime}} p_{\max}(x_j^{\prime}) \right)\det(L_{X^{\prime}})\frac{1}{k}p_{\max}(x_i)}{\left(\prod_{x_j \in X} p_{\max}(x_j) \right)\det(L_X)\frac{1}{k}p_{\max}(x_i^{\prime})}\\
& =
\frac{\left(\prod_{x_j^{\prime} \in X^{\prime} \setminus \{x_i^{\prime}\}} p_{\max}(x_j^{\prime}) \right)\det(L_{X^{\prime}})}{\left(\prod_{x_j \in X \setminus \{x_i\}} p_{\max}(x_j) \right)\det(L_X)} =
\frac{\det(L_{X^{\prime}})}{\det(L_X)}
\end{align}
the last simplification being allowed because $X \setminus \{x_i\} = X^{\prime} \setminus \{x_i^{\prime}\}$.

The only difference from the MH formulation of Theorem \ref{metropolistheorem} is that the support of $R(X^{\prime}|X)$ is not the same of $P_{\text{DPP-TS}}$, as we disallow sampling of batches $X^\prime$ with more than one different element to $X$. However, since $R(X^{\prime}|X) = 0 \Leftrightarrow R(X|X^{\prime}) = 0$, we still satisfy detailed balance in all points. We can then see that $k$ transitions are sufficient to obtain any $X^\prime$ from an existing $X$ when $|X| = k$, and therefore our Markov Chain remains ergodic. With all conditions satisfited, even Algorithm \ref{algomcmc} allows us to approximately sample from the true $P_{\text{DPP-TS}}$.

Algorithm \ref{algomcmc} is a simple modification of an existing k-DPP sampler proposed by \citet{anari_monte_2016}. Furthermore, \citet{rezaei_polynomial_2019} introduces a similar MCMC algorithm for continuous domain DPPs that performs efficiently under certain conditions. Because of its simplicity and effectiveness, Algorithm \ref{algomcmc} is the one we use in all our experiments.

\subsection{Additional Gibbs samplers}

To further demonstrate the simplicity of converting existing k-DPP samplers to $P_{\text{DPP-TS}}$ samplers, we modify \citet{li_fast_2016}'s  algorithm to sample from our $P_{\text{DPP-TS}}$.

\begin{algorithm}[H]
\caption{Modified Gibbs sampler from \citet{li_fast_2016}}
\begin{algorithmic}
\State pick random initial batch $X$
\Repeat
	\State Sample $b$ from uniform Bernoulli distribution
	\If {$b = 1$}
		\State uniformly pick point $x_i \in X$ to replace
		\State sample candidate point $x_i^{\prime} \sim p_{\max}(x_i^{\prime})$
		\State define $X^{\prime} = X \setminus \{x_i\} \cup \{x_i^{\prime}\}$
		\State accept with probability $\alpha = \frac{\det(L_{X^{\prime}})}{\det(L_{X^{\prime}}) + \det(L_{X})}$
		\If {accepted}
			\State $X$ = $X^{\prime}$
		\EndIf
	\EndIf
\Until converged
\end{algorithmic}
\end{algorithm}

Repeating the same steps used for the other Single-point proposal MCMC sampler, we can check that this also satisfies detailed balance. Overall, the procedure is very similar to our preferred Algorithm \ref{algomcmc}.

\section{MUTUAL INFORMATION AND EXPERIMENTAL DESIGN}\label{infobackground}

Modern theoretical analyses of Bayesian Optimization algorithms such as that of \citet{srinivas_gaussian_2010} make use of Mutual Information and other information-theoretic quantities related to $f$. A comprehensive definition of such concepts is also required for our regret analysis.

The main quantity of interest in the aforementioned analysis is indeed the \textit{Mutual Information} $I(f; \mathbf{y}_{1:T})$ between $f$ and a set of observations $\mathbf{y}_{1:T}$ from points $X_{1:T} = \lbrace x_1, \ldots, x_T\rbrace$, sometimes referred to as \textit{Information Gain}. This measures the amount of information learned about the function $f$ by observing $\mathbf{y}_{1:T}$, and for a GP it can be written as
\begin{align}
I\left(f; \mathbf{y}_{1:T}\right) = H\left(\mathbf{y}_{1:T}\right) - H\left(\mathbf{y}_{1:T} \vert f\right) & = \frac{1}{2} \sum_{t=1}^T \log\left( 1 + \sigma^{-2} \sigma^2_t \left( x_t \right) \right)\\
& = \frac{1}{2} \log\det\left( I + \sigma^{-2} K_{X_{1:T}} \right)
\end{align}
where $H\left(\mathbf{y}_{1:T}\right)$ is the differential entropy of the distribution over observations $\mathbf{y}_{1:T}$, $H\left(\mathbf{y}_{1:T} \vert f\right)$ is the differential entropy of the observations conditioned on $f$, $\sigma^2_t(x_t)$ is the posterior variance over $f(x_t)$ conditioned on the partial observations $\mathbf{y}_{1:t-1}$, and $K_{X_{1:T}} = K(X_{1:T}, X_{1:T})$ is the kernel matrix for the prior GP.

The \textit{Conditional Mutual Information} between $f$ and observations $\mathbf{y}_{t:T}$ given previous observations $\mathbf{y}_{1:t-1}$ is then
\begin{align}
I\left(f; \mathbf{y}_{t:T} \vert \mathbf{y}_{1:t-1}\right) & = H\left(\mathbf{y}_{t:T} \vert \mathbf{y}_{1:t-1}\right) - H\left(\mathbf{y}_{t:T} \vert f, \mathbf{y}_{1:t-1}\right)\\
& = H\left(\mathbf{y}_{t:T} \vert \mathbf{y}_{1:t-1}\right) - H\left(\mathbf{y}_{t:T} \vert f\right)\\
& = \frac{1}{2} \sum_{t^\prime=t}^T \log\left( 1 + \sigma^{-2} \sigma^2_{t^\prime} \left( x_{t^\prime} \right) \right) = \frac{1}{2} \log\det\left( I + \sigma^{-2} {K_t}_{X_{t:T}} \right)
\end{align}
with $K_t$ corresponding to the kernel matrix for the posterior kernel $k_t$ of the GP conditioned on the observations $\mathbf{y}_{1:t-1}$.

Mutual Information satisfies the property of \textit{submodularity}, meaning that the information gain $I(f;\mathbf{y}_X \vert \mathbf{y}_{1:t})$ over $f$ of observations $\mathbf{y}_X$ incurs diminishing returns when conditioned on more and more samples $\mathbf{y}_{1:t}$. Essentially, this means that $I(f;\mathbf{y}_X \vert \mathbf{y}_{1:t}) \geq I(f;\mathbf{y}_X \vert \mathbf{y}_{1:t^\prime})$ for any $t^\prime > t$. The most information any set of observations $\mathbf{y}_X$ is able to obtain on $f$ would be at the very beginning $I(f;\mathbf{y}_X)$, not conditioned on any previous examples. Likewise, observing specific additional data will never increase the information any future samples will obtain.

Bayesian Optimization bounds often employ the \textit{Maximum Information Gain} $\gamma_T$ with respect to $f$ obtainable from any observation set $\mathbf{y}_{X}$ of size at most $T$:
\begin{equation}
\gamma_T = \max_{X \in \mathcal{X}, |X| \leq T} I(f;\mathbf{y}_X) \text{.}
\end{equation}
This quantity also bounds any conditional information gain, as by submodularity $I(f;\mathbf{y}_X \vert \mathbf{y}_{1:t}) \leq I(f;\mathbf{y}_X)$, as previously discussed.

\section{NOVEL REGRET BOUNDS PROOFS}
We recall the definition of Bayesian Batch Cumulative Regret for a generic Bayesian Optimization algorithm:
\begin{definition}[Bayes Batch Cumulative Regret]\label{bayesbatchreg}
\emph{
\begin{equation}
{\text{BBCR}_{\text{algo}}}_{T,B} = \mathbb{E}\bigg[ \sum_{t=1}^T \min_{b \in [1,B]} {r_{\text{algo}}}_{t,b} \bigg] = \mathbb{E}\bigg[ \sum_{t=1}^T \min_{b \in [1,B]} \left( f(x^{\star}) - f(x_{\text{algo}\; t,b}) \right) \bigg] \text{.}
\end{equation}
}
\end{definition}
We make use of proof techniques from \citet{contal_parallel_2013} and \citet{russo_learning_2014} to prove a bound on Bayesian Batch Cumulative Regret for TS and DPP-TS equivalent to one obtainable by sequential full-feedback (non batched) TS, and by consequence a bound on Bayesian Simple Regret.

\subsection{The BBCR Bound Proof for TS}\label{tsbbcrboundproof}
Before proving the bound for the novel DPP-TS, we do so for regular Batched Thompson Sampling.

We first recall a few statements from \citet{russo_learning_2014}, necessary to justify subsequent steps in our proof. If given the current posterior $\text{GP}(\mu_t,K_t)$, we define the Upper Confidence Bound $U_t(x) = \mu_t(x) + \sqrt{\beta_t}\sigma_t(x)$ for any $\beta_t$ exactly as in GP-UCB, we can show the following:
\begin{proposition}[\citet{russo_learning_2014}]\label{russotsucbprop}
For any $U_t$ sequence defined by some $\beta_t$ sequence
\emph{
\begin{equation}
{\text{BCR}_{\text{TS }}}_T = \mathbb{E}\bigg[ \sum_{t=1}^T \left( U_t({x_{\text{TS}}}_t) - f({x_{\text{TS}}}_t) \right) \bigg] + \mathbb{E}\bigg[ \sum_{t=1}^T \left( f(x^\star) - U_t(x^\star) \right) \bigg] \text{.}
\end{equation}
}
\end{proposition}
This can be shown by first rewriting 
\begin{align}
\mathbb{E}\left[ \sum_{t=1}^T \left( f(x^\star) - f({x_{\text{TS}}}_t) \right) \right] = \mathbb{E}\left[ \sum_{t=1}^T \left( f(x^\star) - U_t({x_{\text{TS}}}_t) + U_t({x_{\text{TS}}}_t) - f({x_{\text{TS}}}_t) \right) \right]\\
= \mathbb{E}\left[ \sum_{t=1}^T \left( U_t({x_{\text{TS}}}_t) - f({x_{\text{TS}}}_t) \right) \right] + \mathbb{E}\left[ \sum_{t=1}^T \left( f(x^\star) - U_t({x_{\text{TS}}}_t) \right) \right]
\end{align}
and noticing that, conditioned on the history $D_{t-1}$, $x^\star$ and ${x_{\text{TS}}}_t$ are identically distributed and $U_t(.)$ is a deterministic function. Therefore $\mathbb{E}\left[ U_t(x^\star) | D_{t-1} \right] = \mathbb{E}\left[ U_t({x_{\text{TS}}}_t) | D_{t-1} \right]$, and the overall expectation over these terms maintains equality.

\citet{russo_learning_2014} the proceed to bound both components. Assuming a finite domain and using $\beta_t = 2\ln\left(\frac{B(t^2 + 1)|\mathcal{X}|}{\sqrt{2\pi}}\right)$, following them we obtain
\begin{equation}\label{russo1bound}
\mathbb{E}\bigg[ f(x^\star) - U_t(x^\star) \bigg] \leq \frac{1}{B(t^2 + 1)} ,\qquad \mathbb{E}\bigg[ \sum_{t=1}^T \left( f(x^\star) - U_t(x^\star) \right) \bigg] \leq \frac{1}{B}\sum_{t=1}^T \left( \frac{1}{t^2 + 1} \right) \leq \frac{1}{B}
\end{equation}
and
\begin{equation}\label{russo2bound}
\mathbb{E}\bigg[ U_t({x_{\text{TS}}}_t) - f({x_{\text{TS}}}_t) \bigg] =  \mathbb{E}\bigg[ \sqrt{\beta_t} \sigma_t({x_{\text{TS}}}_t) \bigg] \text{.}
\end{equation}

With this established, as a first step in our proof we modify Lemma 1 from \citet{contal_parallel_2013} and introduce
\begin{lemma}\label{tslemma1}
For finite $\mathcal{X}$ and $\beta_t = 2\ln\left(\frac{B(t^2 + 1)|\mathcal{X}|}{\sqrt{2\pi}}\right)$, we have 
\emph{
\begin{equation}
\mathbb{E}\bigg[ \min_{b \in [1,B]} {r_{\text{TS}}}_{t,b} \bigg] \leq \mathbb{E}\bigg[ {r_{\text{TS}}}_{t,1} \bigg] \leq \frac{1}{B(t^2 + 1)} + \mathbb{E}\bigg[ \sqrt{\beta_t} \sigma_{t,1}(x_{\text{TS}\; t,1}) \bigg] \text{.}
\end{equation}
}
\end{lemma}
\begin{proof}
\begin{align}
\mathbb{E}\bigg[ \min_{b \in [1,B]} {r_{\text{TS}}}_{t,b} \bigg] & \leq \mathbb{E}\bigg[ {r_{\text{TS}}}_{t,1} \bigg] = \mathbb{E}\bigg[ f(x^{\star}) - f(x_{\text{TS}\; t,1}) \bigg]\\
\label{tslemma1step1} & = \mathbb{E}\bigg[ f(x^{\star}) - U_{t,1}(x^{\star}) + U_{t,1}(x_{\text{TS}\; t,1}) - f(x_{\text{TS}\; t,1}) \bigg]\\
\label{tslemma1step2} & \leq \frac{1}{B(t^2 + 1)} + \mathbb{E}\bigg[ \sqrt{\beta_t} \sigma_{t,1}(x_{\text{TS}\; t,1}) \bigg] \text{.}
\end{align}
Step (\ref{tslemma1step1}) can be performed due to Russo and Van Roy's Proposition \ref{russotsucbprop}, by adding and subtracting the UCB $U_{t,1}(x_{\text{TS}\; t,1}) = \mu_{t,1}(x_{\text{TS}\; t,1}) + \sqrt{\beta_t}\sigma_{t,1}(x_{\text{TS}\; t,1})$ and noticing that, conditioned on the history $D_{t-1}$, $x^\star$ and ${x_{\text{TS}}}_{t,1}$ are identically distributed and $U_{t,1}(.)$ is a deterministic function. Therefore $\mathbb{E}\left[ U_{t,1}(x^\star) | D_{t-1,B} \right] = \mathbb{E}\left[ U_{t,1}({x_{\text{TS}}}_{t,1}) | D_{t-1,B} \right]$.

To obtain step (\ref{tslemma1step2}), we then separate the terms from Equation (\ref{tslemma1step1}) into $\mathbb{E}\left[ f(x^{\star}) - U_{t,1}(x^{\star})\right]$ and $\mathbb{E}\left[ U_{t,1}(x_{\text{TS}\; t,1}) - f(x_{\text{TS}\; t,1}) \right]$, bounding the first with Equation (\ref{russo1bound}) and the second with Equation (\ref{russo2bound}).
\end{proof}
\*~

After proving this essential Lemma, we proceed with adapting Lemma 2 from \citet{contal_parallel_2013}. Unlike them, we need not bother with guarantees about a maximizer high probability region $\mathfrak{R}^+$ as the one defined for GP-UCB-PE, as we are operating in expectation.

\begin{lemma}\label{tslemma2}
In expectation, the deviation of the first point within a batch selected by TS is bounded by the one for any point within the previous batch selected by TS, thus
\emph{
\begin{equation}
\mathbb{E}\bigg[ \sigma_{t+1,1}({x_{\text{TS }}}_{t+1,1}) \bigg] \leq \mathbb{E}\bigg[ \sigma_{t,b}({x_{\text{TS }}}_{t,b}) \bigg] \qquad \forall t \in [1,T-1], \; \forall b \in [1,B] \text{.}
\end{equation}
}
\end{lemma}

\begin{proof}
For any time $t$, for every step $(t,b)$ within the batch, the points ${x_{\text{TS }}}_{t,b}$ for TS are independently sampled from $P_{\max\; t,1}$, which depends on history up to $D_{t-1,B}$. Therefore, given $D_{t,b-1}$, $\sigma_{t,b}(x)$ is a deterministic function, and ${x_{\text{TS }}}_{t,b}$ and the true $x^{\star}$ have the same distribution. We thus have that $\forall t \in [1,T], \; \forall b \in [1,B]$:
\begin{align} \label{tslemma2eq1}
\mathbb{E}\bigg[ \sigma_{t,b}({x_{\text{TS }}}_{t,b}) \bigg] & = \mathbb{E}\bigg[ \mathbb{E}\bigg[ \sigma_{t,b}({x_{\text{TS }}}_{t,b}) \bigg| D_{t,b-1} \bigg] \bigg]\\ & = \mathbb{E}\bigg[ \mathbb{E}\bigg[ \sigma_{t,b}(x^{\star}) \bigg| D_{t,b-1} \bigg] \bigg] = \mathbb{E}\bigg[ \sigma_{t,b}(x^{\star}) \bigg]
\end{align}
Because of the law of non-increasing variance \citep{rasmussen_gaussian_2005}, we have that
\begin{equation} \label{tslemma2eq2}
\mathbb{E}\bigg[ \sigma_{t+1,1}(x^{\star}) \bigg] \leq \mathbb{E}\bigg[ \sigma_{t,b}(x^{\star}) \bigg] \qquad \forall t \in [1,T-1], \; \forall b \in [1,B]
\end{equation}
and therefore:
\begin{equation}
\mathbb{E}\bigg[ \sigma_{t+1,1}({x_{\text{TS }}}_{t+1,1}) \bigg] \leq \mathbb{E}\bigg[ \sigma_{t,b}({x_{\text{TS }}}_{t,b}) \bigg] \qquad \forall t \in [1,T-1], \; \forall b \in [1,B] \text{.}
\end{equation}
\end{proof}
\*~

We can then introduce
\begin{lemma}\label{tslemma3}
In expectation, the sum of deviations for the first points of all batches selected by TS is bounded by the sum of deviations for all points selected by TS, divided by $B$.
\emph{
\begin{equation}
\mathbb{E}\bigg[ \sum_{t=1}^T \sigma_{t,1}({x_{\text{TS }}}_{t,1}) \bigg] \leq \mathbb{E}\bigg[ \frac{1}{B} \sum_{t=1}^T \sum_{b=1}^B \sigma_{t,b}({x_{\text{TS }}}_{t,b}) \bigg]
\end{equation}
}
\end{lemma}

\begin{proof}
For all $t$, using Lemma \ref{tslemma2} and summing over $b$, we can get
\begin{equation}
\mathbb{E}\bigg[ \sigma_{t,1}({x_{\text{TS }}}_{t,1}) + (B-1)\sigma_{t+1,1}({x_{\text{TS }}}_{t+1,1}) \bigg] \leq \mathbb{E}\left[ \sigma_{t,1}({x_{\text{TS }}}_{t,1}) + \sum_{b=2}^B \sigma_{t,b}({x_{\text{TS }}}_{t,b}) \right]
\end{equation}
Summing both sides over $t$ and dividing by $B$, we obtain the desired result.
\end{proof}
\*~

\begin{lemma}\label{tslemma4}
Assuming without loss of generality that, for all $t$ and $b$, $\left(\sigma_{t,b}({x_{\text{TS }}}_{t,b})\right)^2 \leq 1$, the sum of variances of the points selected by TS is bounded by a constant factor times $\gamma_{TB}$:
\emph{
\begin{equation}
\sum_{t=1}^T \sum_{b=1}^B \left(\sigma_{t,b}({x_{\text{TS }}}_{t,b})\right)^2 \leq C_2 \gamma_{TB}
\end{equation}
}
with $C_2 = 2/\log(1 + \sigma^{-2})$ and $\gamma_{TB}$ being the maximum information gain on $f$ from $TB$ observations as defined in Appendix \ref{infobackground}.
\end{lemma}

\begin{proof}
The information gain on $f$ from a sequence of $TB$ observations can be expressed in terms of the posterior variances
\begin{equation}
I\left( f(x_{1:T,1:B}); \mathbf{y}_{1:T,1:B}\right) = \frac{1}{2}\sum_{t=1}^T \sum_{b=1}^B \log\left(1 + \sigma^{-2}\left(\sigma_{t,b}(x_{t,b})\right)^2\right)
\end{equation}
as seen in Appendix \ref{infobackground}, and is bounded by $\gamma_{TB}$ by definition. We can then obtain, thanks to the bounded variance assumption:
\begin{align}
\sum_{t=1}^T \sum_{b=1}^B \left(\sigma_{t,b}({x_{\text{TS }}}_{t,b})\right)^2 & \leq \sum_{t=1}^T \sum_{b=1}^B \frac{1}{\log(1 + \sigma^{-2})} \log\left(1 + \sigma^{-2}\left(\sigma_{t,b}({x_{\text{TS }}}_{t,b})\right)^2\right)\\
& = \frac{2}{\log(1 + \sigma^{-2})} \, I\left( f({x_{\text{TS }}}_{1:T,1:B}); \mathbf{y}_{1:T,1:B}\right) \text{.}
\end{align}
\end{proof}
\*~

Finally, we can conclude by introducing our Bayesian Batch Cumulative Regret bound.
\begin{theorem}[Bayes Batch Cumulative Regret Bound for Batched Thompson Sampling]\label{tsbatchbound-2}
If $f \sim GP(0,K)$ with covariance kernel bounded by 1 and noise model $\mathcal{N}(0,\sigma^2)$, and either
\begin{itemize}
\item Case 1: finite $\mathcal{X}$ and $\beta_t = 2\ln\left(\frac{B(t^2 + 1)|\mathcal{X}|}{\sqrt{2\pi}}\right)$;
\item Case 2: compact and convex $\mathcal{X} \in [0,l]^d$, with Assumption \ref{gradassumpt} satisfied and $\beta_t = 4(d+1)\log(Bt) + 2d\log(dab\sqrt{\pi})$.
\end{itemize}
Then Batched Thompson Sampling attains Bayes Batch Cumulative Regret of
\emph{
\begin{equation}
{\text{BBCR}_{\text{TS}}}_{T,B} \leq \frac{C_1}{B} + \sqrt{ C_2 \frac{T}{B} \beta_T \gamma_{TB}}
\end{equation}
}
with $C_1 = 1$ for Case 1, $C_1 = \frac{\pi^2}{6} + \frac{\sqrt{2\pi}}{12}$ for Case 2, and $C_2 = \frac{2}{\log(1 + \sigma^{-2})}$.
\end{theorem}

\begin{proof}
Using the previous lemmas together with Russo and Van Roy inequalities, we can show for Case 1:
\begin{align}
{\text{BBCR}_{\text{TS}}}_{T,B} & = \mathbb{E}\bigg[ \sum_{t=1}^T \min_{b \in [1,B]} {r_{\text{TS}}}_{t,b} \bigg] \leq \mathbb{E}\bigg[ \sum_{t=1}^T {r_{\text{TS}}}_{t,1} \bigg]\\
& \leq \sum_{t=1}^T \frac{1}{B(t^2 + 1)} + \mathbb{E}\bigg[ \sum_{t=1}^T \sqrt{\beta_t} \sigma_{t,1}({x_{\text{TS}}}_{t,1}) \bigg] \quad \text{by Lemma \ref{tslemma1}}\\
& \leq \frac{C_1}{B} + \mathbb{E}\bigg[ \sqrt{\beta_T} \frac{1}{B} \sum_{t=1}^T \sum_{b=1}^B \sigma_{t,b}(x_{t,b}) \bigg]\quad \text{by Eq. (\ref{russo1bound}) and Lemma \ref{tslemma3}}\\
& \leq \frac{C_1}{B} + \mathbb{E}\bigg[ \sqrt{\beta_T} \frac{1}{B} \sqrt{ TB \sum_{t=1}^T \sum_{b=1}^B (\sigma_{t,b}(x_{t,b}))^2} \bigg] \quad \text{by Cauchy-Schwartz}\\
& \leq \frac{C_1}{B} + \sqrt{ C_2 \frac{T}{B} \beta_T \gamma_{TB}}\quad \text{by Lemma \ref{tslemma4}}
\end{align}
For Case 2, we simply modify the steps of Lemma \ref{tslemma1} with the corresponding inequalities used by \citet{kandasamy_parallelised_2018} for their continuous-domain version of the bound.
\end{proof}
\*~

The bound we just derived scales equivalently to the bound obtainable by standard sequential TS with full feedback. This is in contrast to the bound previously obtained by \citet{kandasamy_parallelised_2018} for Batched TS without initialization:
\begin{equation}
{\text{BBCR}_{\text{TS}}}_{T,B} \leq \frac{C_1}{B} + \sqrt{C_2 \frac{T}{B} \exp(C) \beta_T \gamma_{TB}} 
\end{equation}
which depends on an additional factor $\exp(C)$ dependent on $B$, rendering the bound not convergent in $B$ unless a wasteful initialization procedure is performed before TS.

\subsection{The BBCR Bound Proof for DPP-TS}\label{dpptsbbcrboundproof}
We now consider DPP-TS again, and strive to obtain an equivalent bound.

In order to do so, we must modify the algorithm we developed and introduce DPP-TS-alt, so that for every batch:
\begin{itemize}
\item For the first sample in the batch $x_{\text{DPP-TS-alt}\; t,1}$, we sample from $p_{\max t}$ as in standard Thompson Sampling;
\item For all the other samples $x_{\text{DPP-TS-alt}\; t,b}$ with $b \in [2,B]$, we sample from joint $P_{\text{DPP-TS }t}$, using the most updated posterior variance matrix $K_{t,1}$ to define the DPP kernel.
\end{itemize}

We begin by noting that Lemma \ref{tslemma1} is applicable to DPP-TS-alt as well, as $x_{\text{DPP-TS-alt}\; t,1}$ behaves exactly in the same way as $x_{\text{TS}\; t,1}$, since DPP-TS-alt has been explicitly defined as using standard Thompson Sampling for the first sampled point of every batch.

Then, we must translate Lemma \ref{tslemma2} to DPP-TS-alt as well, which requires a more involved proof. In fact, we must split the undertaking in three preliminary lemmas.

First, when considering the batch sampled at time $t$, we introduce for sake of argument an \textit{alternative point} $\tilde{x}_{t,B}$ to take the place of the last sampled point of the batch $x_{\text{DPP-TS-alt}\; t,B}$. The original point $x_{\text{DPP-TS-alt}\; t,B}$ is sampled from the DPP, and is distributed as $x_{\text{DPP-TS-alt}\; t,B} \sim P_{\text{DPP-TS }t}(x_{t,B} | x_{t,1}, \ldots, x_{t,B-1}) = p_{\text{DPP-TS}\; t,B}(x_{t,B})$ when conditioned on the previous points of the batch. Instead, we define the replacement as $\tilde{x}_{t,B} \sim p_{\max\; t,1}$, sampled from standard Thompson Sampling with the posterior available from observations up to $(t,1)$.
\begin{lemma}\label{dpptslemma2.1}
At time $t$, let $x_{\text{DPP-TS-alt}\; t,B}$ be the last point of the batch chosen by DPP-TS-alt. Let us in its place define an alternative element $\tilde{x}_{t,B}$ obtained by following the DPP-TS-alt procedure up to step $(t,B-1)$ and then sampling using regular TS from the available maximizer posterior distribution $p_{\max\; t,1}$ instead of from conditioned $p_{\text{DPP-TS}\; t,B}$.

We can then show that 
\emph{
\begin{equation}
\mathbb{E}\bigg[ \sigma_{t+1,1}({x_{\text{DPP-TS-alt }}}_{t+1,1}) \bigg] \leq \mathbb{E}\bigg[ \sigma_{t,B}(\tilde{x}_{t,B}) \bigg] \text{.}
\end{equation}
}
\end{lemma}

\begin{proof}
Equations (\ref{tslemma2eq1}) and (\ref{tslemma2eq2}) are still valid for $\tilde{x}_{t,B}$ (being sampled from TS), as conditioned on history $D_{t,B-1}$, $\tilde{x}_{t,B}$ has the same distribution of $x^\star$, and so we obtain that
\begin{align} \label{lemma2dppeq1}
\mathbb{E}\bigg[ \sigma_{t+1,1}({x_{\text{DPP-TS-alt }}}_{t+1,1}) \bigg] & = \mathbb{E}\bigg[ \sigma_{t+1,1}(x^{\star}) \bigg]\\
& \leq \mathbb{E}\bigg[ \sigma_{t,B}(x^{\star}) \bigg] = \mathbb{E}\bigg[ \mathbb{E}\bigg[ \sigma_{t,B}(\tilde{x}_{t,B}) \bigg| D_{t,B-1} \bigg] \bigg]
\end{align}
\end{proof}

\begin{lemma}\label{dpptslemma2.2}
Given $\tilde{x}_{t,B}$ defined as in Lemma \ref{dpptslemma2.1}, we have that
\emph{
\begin{equation}
\mathbb{E}\bigg[ \sigma_{t,B}(\tilde{x}_{t,B}) \bigg] \leq \mathbb{E}\bigg[ \sigma_{t,B}({x_{\text{DPP-TS-alt }}}_{t,B}) \bigg] \text{.}
\end{equation}
}
\end{lemma}

\begin{proof}
To prove the lemma, we first observe that (from Appendix \ref{infobackground})
\begin{align}
    \det\left( I + \sigma^{-2} {K_t}_{X_{1:B}} \right) = \prod_{b=1}^B \left( 1 + \sigma^{-2} \sigma^2_{b} \left( x_{b} \right) \right) \text{.}
\end{align}
We then obtain the marginal distribution of the last point of a DPP-TS batch by summing over the domain:
\begin{align}
    p_{\text{DPP-TS }t,B}(x_{t,B}) & = \sum_{(x_{t,1},\ldots,x_{t,B-1}) \in \mathcal{X}^{B-1}} P_{\text{DPP-TS }t}\left( x_{t,1},\ldots,x_{t,B-1},x_{t,B} \right)\\
    & \propto \sum_{(x_{t,1},\ldots,x_{t,B-1}) \in \mathcal{X}^{B-1}} \left( \left(\prod_{b=1,\ldots,B} p_{\max t,1}(x_{t,b}) \right) \det\left( I + \sigma^{-2} {K_t}_{X_{1:B}} \right) \right)\\
    & \propto \sum_{(x_{t,1},\ldots,x_{t,B-1}) \in \mathcal{X}^{B-1}} \left( \prod_{b=1,\ldots,B} p_{\max t,1}(x_{t,b}) \left( 1 + \sigma^{-2} \sigma^2_{b} \left( x_{t,b} \right) \right)\right)\\
    & = p_{\max t,1}(x_{t,B})\\
    &\qquad \left( 1 + \sigma^{-2} \sigma^2_{B} \left( x_{t,B} \right) \right) \sum_{(x_{t,1},\ldots,x_{t,B-1}) \in \mathcal{X}^{B-1}} \left( \prod_{b=1,\ldots,B-1} p_{\max t,1}(x_{t,b}) \left( 1 + \sigma^{-2} \sigma^2_{b} \left( x_{t,b} \right) \right) \right)\\
    & \propto p_{\max t,1}(x_{t,B}) \left( 1 + \sigma^{-2} \sigma^2_{B} \left( x_{t,B} \right) \right) \text{.}
\end{align}
We can then show that
\begin{align}
p_{\text{DPP-TS }t,B}(x) & = \frac{1 + \sigma^{-2}\sigma_{t,B}^2(x)}{\sum_{x^{\prime} \in \mathcal{X}} p_{\max t,1}(x^{\prime}) (1 + \sigma^{-2}\sigma_{t,B}^2(x^{\prime}))} p_{\max t,1}(x)\\
& = \frac{1 + \sigma^{-2}\sigma_{t,B}^2(x)}{\mathbb{E}_{p_{\max t,1}}\left[ 1 + \sigma^{-2}\sigma_{t,B}^2(x^{\prime}) \right]} p_{\max t,1}(x)\\
& = (1 + \delta(x)) p_{\max t,1}(x)
\end{align}
with $\delta(x) = \frac{\sigma^{-2}\sigma_{t,B}^2(x) - \mathbb{E}_{P_{\max t,1}}\left[ \sigma^{-2}\sigma_{t,B}^2(x^{\prime}) \right]}{\mathbb{E}_{P_{\max t,1}}\left[ 1 + \sigma^{-2}\sigma_{t,B}^2(x^{\prime}) \right]}$.

As both $p_{\max t,1}$ and $ p_{\text{DPP-TS }t,B}$ are distributions, we have that
\begin{align}
\sum_{x \in \mathcal{X}} p_{\max t,1}(x) & = 1\\
\sum_{x \in \mathcal{X}} p_{\text{DPP-TS }t,B}(x)  &= \sum_{x \in \mathcal{X}} (1 + \delta(x)) p_{\max t,1}(x) = 1
\end{align}
and therefore
\begin{align}
\sum_{x \in \mathcal{X}} \delta(x) p_{\max t,1}(x) = 0
\end{align}

We can rewrite $\mathbb{E}_{p_{\text{DPP-TS }t,B}}\left[ \sigma_{t,B}(x) \right]$ as
\begin{align}
\mathbb{E}_{p_{\text{DPP-TS }t,B}}\bigg[ \sigma_{t,B}(x) \bigg] & = \sum_{x \in \mathcal{X}} (1 + \delta(x)) p_{\max t,1}(x) \sigma_{t,B}(x)\\
\label{dpptslemma2.2step}& = \mathbb{E}_{p_{\max t,1}}\bigg[ \sigma_{t,B}(x) \bigg] + \sum_{x \in \mathcal{X}} \delta(x) p_{\max t,1}(x) \sigma_{t,B}(x)
\end{align}
and knowing that by definition of $\delta(x)$
\begin{align}
\delta(x) \geq 0 \quad \Leftrightarrow & \quad \sigma^{-2}\sigma_{t,B}^2(x) \geq \mathbb{E}_{p_{\max t,1}}\left[ \sigma^{-2}\sigma_{t,B}^2(x^{\prime}) \right]\\
& \quad \sigma_{t,B}(x) \geq \sqrt{\mathbb{E}_{p_{\max t,1}}\left[ \sigma_{t,B}^2(x^{\prime}) \right]}
\end{align}
we can see that
\begin{equation}
\begin{split}
&\sum_{x \in \mathcal{X}} \delta(x) p_{\max t,1}(x) \sigma_{t,B}(x)\\
& = \sum\limits_{\substack{x \in \mathcal{X}\\ \delta(x) \geq 0}} \delta(x) p_{\max t,1}(x) \sigma_{t,B}(x) + \sum\limits_{\substack{x \in \mathcal{X}\\ \delta(x) < 0}} \delta(x) p_{\max t,1}(x) \sigma_{t,B}(x)\\
&\geq \left(\sum_{x \in \mathcal{X}} \delta(x) p_{\max t,1}(x)\right) \sqrt{\mathbb{E}_{p_{\max t,1}}\left[ \sigma_{t,B}^2(x^{\prime}) \right]} = 0 \text{.}
\end{split}
\end{equation}
We combine this with Equation (\ref{dpptslemma2.2step}) to prove that
\begin{equation} \label{lemma2dppeq2}
\mathbb{E}\bigg[ \mathbb{E}\bigg[ \sigma_{t,B}(\tilde{x}_{t,B}) \bigg| D_{t,B-1} \bigg] \bigg] \leq \mathbb{E}\bigg[ \mathbb{E}\bigg[ \sigma_{t,B}({x_{\text{DPP-TS-alt }}}_{t,B}) \bigg| D_{t,B-1} \bigg] \bigg] \text{,}
\end{equation}
hence $p_{\text{DPP-TS }t,B}$, the conditional distribution of the last point of the batch, is a reweigthing of $p_{\max t,1}$ in which more probability mass is put on points with higher posterior variance $\sigma_{t,B}^2(x)$.
\end{proof}

\begin{lemma}\label{dpptslemma2.3}
For any timestep $t$ and for all $b \in [3,B]$, we have that
\emph{
\begin{equation}
\mathbb{E}\bigg[ \sigma_{t,b}({x_{\text{DPP-TS-alt }}}_{t,b}) \bigg] \leq \mathbb{E}\bigg[  \sigma_{t,b-1}({x_{\text{DPP-TS-alt }}}_{t,b-1}) \bigg] \text{.}
\end{equation}
}
\end{lemma}

\begin{proof}
Given any timestep $t$ and any $b \in [3,B]$, we first condition on history $D_{t,b-2}$, and take $P_{\text{DPP-TS }t,b-1}$ as the conditional joint distribution for selecting the remaining points of the batch with $b^\prime \in [b-1,B]$ for DPP-TS-alt. Because of the conditioning, the distribution is defined on deterministic quantities only.

We can see that $P_{\text{DPP-TS }t,b-1}(X) \propto P_{\max t,1}(X) \det({L_{t,b-1}}_X)$ (defined following a similar argument as that at the beginning of the proof for Lemma \ref{dpptslemma2.2}), rewritten as a joint distribution $P_{\text{DPP-TS }t,b-1}(x_{b-1}, \ldots, x_B)$, has the property of exchangeability with respect to points of the batch $X = (x_{b-1}, \ldots, x_B)$, meaning that changing the order of the points within the batch will not affect the probability $P_{\text{DPP-TS }t,b-1}(X)$. This is true as $P_{\max t,1}(X) = \prod_{b^\prime =b-1}^B p_{\max t,1}(x_b^\prime)$ depends on each $x_b^\prime$ independently, and the DPP term $\det({L_{t,b-1}}_X)$ is clearly exchangeable for any valid kernel $L_{t,b-1}$, regardless of regularization by $I$.

Therefore, any two points ${x_{\text{DPP-TS-alt }}}_{t,b}$ and ${x_{\text{DPP-TS-alt }}}_{t,b-1}$ within the batch have the same marginal distribution. This is true because of exchangeability, as integrating all other points must lead to the same result regardless of the position of the marginalized point within the batch. Furthermore, $\sigma_{t,b-1}$ is a deterministic function given $D_{t,b-2}$, and so:
\begin{equation}\label{dpptslemma2.3step1}
\mathbb{E}\bigg[ \sigma_{t,b-1}({x_{\text{DPP-TS-alt }}}_{t,b}) \bigg| D_{t,b-2} \bigg] = \mathbb{E}\bigg[ \sigma_{t,b-1}({x_{\text{DPP-TS-alt }}}_{t,b-1}) \bigg| D_{t,b-2} \bigg] \text{.}
\end{equation}

By the law of non-decreasing variance \citep{rasmussen_gaussian_2005}, we also have that
\begin{equation}\label{dpptslemma2.3step2}
\mathbb{E}\bigg[ \sigma_{t,b}({x_{\text{DPP-TS-alt }}}_{t,b}) \bigg| D_{t,b-2} \bigg] \leq \mathbb{E}\bigg[ \sigma_{t,b-1}({x_{\text{DPP-TS-alt }}}_{t,b}) \bigg| D_{t,b-2} \bigg] \text{.}
\end{equation}

Combining Equations (\ref{dpptslemma2.3step1}) and (\ref{dpptslemma2.3step2}), we finally obtain
\begin{equation}
\mathbb{E}\bigg[ \sigma_{t,b}({x_{\text{DPP-TS-alt }}}_{t,b}) \bigg| D_{t,b-2} \bigg] \leq \mathbb{E}\bigg[ \sigma_{t,b-1}({x_{\text{DPP-TS-alt }}}_{t,b-1}) \bigg| D_{t,b-2} \bigg] \text{,}
\end{equation}
and taking the overall expectation over both sides yields the lemma.
\end{proof}

Finally, we can now obtain the promised lemma:
\begin{lemma}\label{dpptslemma2final}
In expectation, when using DPP-TS-alt, the deviation of the first point of a batch, selected by standard TS, is bounded by the one for any point within the previous batch, selected by DPP-TS, thus
\emph{
\begin{equation}
\mathbb{E}\bigg[ \sigma_{t+1,1}({x_{\text{DPP-TS-alt }}}_{t+1,1}) \bigg] \leq \mathbb{E}\bigg[ \sigma_{t,b}({x_{\text{DPP-TS-alt }}}_{t,b}) \bigg] \quad \forall t \in [1,T-1], \; \forall b \in [2,B] \text{.}
\end{equation}
}
\end{lemma}

\begin{proof}
Combine Lemmas \ref{dpptslemma2.1}, \ref{dpptslemma2.2} and \ref{dpptslemma2.3} to obtain the inequality.
\end{proof}
\*~

Lemma \ref{tslemma3} from the TS proof follows from lemma \ref{dpptslemma2final}, and lemma \ref{tslemma4} applies to the DPP-TS-alt as well without any further adjustments. It's then a matter of combining the new results to get the final bound.

\begin{theorem}[Bayes Batch Cumulative Regret Bound for DPP-TS]\label{dpptsbatchbound-2}
If $f \sim GP(0,K)$ with covariance kernel bounded by 1 and noise model $\mathcal{N}(0,\sigma^2)$, and either
\begin{itemize}
\item Case 1: finite $\mathcal{X}$ and $\beta_t = 2\ln\left(\frac{B(t^2 + 1)|\mathcal{X}|}{\sqrt{2\pi}}\right)$;
\item Case 2: compact and convex $\mathcal{X} \in [0,l]^d$, with Assumption \ref{gradassumpt} satisfied and $\beta_t = 4(d+1)\log(Bt) + 2d\log(dab\sqrt{\pi})$.
\end{itemize}
Then DPP-TS (in its DPP-TS-alt variant) attains Bayes Batch Cumulative Regret of
\emph{
\begin{equation}
{\text{BBCR}_{\text{DPP-TS}}}_{T,B} \leq \frac{C_1}{B} + \sqrt{ C_2 \frac{T}{B} \beta_T \gamma_{TB}} - C_3
\end{equation}
}
with $C_1 = 1$ for Case 1, $C_1 = \frac{\pi^2}{6} + \frac{\sqrt{2\pi}}{12}$ for Case 2, $C_2 = \frac{2}{\log(1 + \sigma^{-2})}$ and $-C_3 \leq 0$.
\end{theorem}

\begin{proof}
Using the previous lemmas together with Russo and Van Roy inequalities, we can show for Case 1:
\begin{align}
{\text{BBCR}_{\text{DPP-TS}}}_{T,B} & = \mathbb{E}\bigg[ \sum_{t=1}^T \min_{b \in [1,B]} {r_{\text{DPP-TS}}}_{t,b} \bigg] \leq \mathbb{E}\bigg[ \sum_{t=1}^T {r_{\text{DPP-TS}}}_{t,1} \bigg]\\
& \leq \sum_{t=1}^T \frac{1}{B(t^2 + 1)} + \mathbb{E}\bigg[ \sum_{t=1}^T \sqrt{\beta_t} \sigma_{t,1}({x_{\text{DPP-TS}}}_{t,1}) \bigg] \quad \text{by Lemma \ref{tslemma1}}\\
& \leq \frac{C_1}{B} + \mathbb{E}\bigg[ \sqrt{\beta_T} \frac{1}{B} \sum_{t=1}^T \sum_{b=1}^B \sigma_{t,b}(x_{t,b}) \bigg]\quad \text{by Eq. (\ref{russo1bound}) and Lemma \ref{tslemma3}}\\
& \leq \frac{C_1}{B} + \mathbb{E}\bigg[ \sqrt{\beta_T} \frac{1}{B} \sqrt{ TB \sum_{t=1}^T \sum_{b=1}^B (\sigma_{t,b}(x_{t,b}))^2} \bigg] \quad \text{by Cauchy-Schwartz}\\
& \leq \frac{C_1}{B} + \sqrt{ C_2 \frac{T}{B} \beta_T \gamma_{TB}}\quad \text{by Lemma \ref{tslemma4}}
\end{align}
For Case 2, we again simply modify the steps of Lemma \ref{tslemma1} with the corresponding inequalities used by \citet{kandasamy_parallelised_2018} for their continuous-domain version of the bound.
\end{proof}

We have shown that the Bayesian bound for DPP-TS is at least as good as that of standard Batched TS. In fact, the bound is even better than for that for TS, as we can add negative factors $-{C_3}_t = -\mathbb{E}\left[ \delta(x) \sigma_{t,B}(x) \right]$ at every iteration $t$, leftover from Equation \ref{dpptslemma2.2step} in lemma \ref{dpptslemma2.2}. From this we obtain the negative factor $-C_3$.

\section{DISCUSSION OF KATHURIA ET AL.\ (\citeyear{kathuria_batched_2016})'S BOUND}\label{appendix:kathuriabound}
Despite their meaningful insight and better practical performance compared to \citet{contal_parallel_2013}'s original GP-UCB-PE, \citet{kathuria_batched_2016}'s proposed regret bound for UCB-DPP-SAMPLE does not improve on existing bounds, as it is founded on bounding the expected information gain from the last $k = B-1$ points of every batch:
\begin{align}
& \mathbb{E}_{S \sim k\text{-DPP}(L_{t,1})}\left[ \log \det ((L_{t,1})_S) \right]\\
& = \sum_{|S|=k}\frac{\det((L_{t,1})_S) \log(\det ((L_{t,1})_S))}{\sum_{|S|=k}\det((L_{t,1})_S)}\\
& = \sum_{|S|=k}\frac{\det((L_{t,1})_S) \log(\frac{\det ((L_{t,1})_S)}{\sum_{|S|=k}\det((L_{t,1})_S)})}{\sum_{|S|=k}\det((L_{t,1})_S)} + \sum_{|S|=k}\frac{\det((L_{t,1})_S) \log(\sum_{|S|=k}\det ((L_{t,1})_S))}{\sum_{|S|=k}\det((L_{t,1})_S)}\\
& = -H(k\text{-DPP}(L_{t,1})) + \log\left(\sum_{|S|=k}\det ((L_{t,1})_S)\right)\\
& \leq -H(k\text{-DPP}(L_{t,1})) + \log\left(|\mathcal{X}|^k \max\det ((L_{t,1})_S) \right)\\
& \leq -H(k\text{-DPP}(L_{t,1})) + k\log\left(|\mathcal{X}|\right) + \log\left(\max\det ((L_{t,1})_S)\right)
\end{align}
When summing this over all iterations $T$, the last term is bounded by some $C^\prime \gamma_{TB}$ with $C^\prime \geq 1$, while the first two terms summed together are trivially positive, as $H(k\text{-DPP}(L_{t,1})) \leq  k\log\left(|\mathcal{X}|\right)$. We then have that $\sum_{t=1}^T{\left(-H(k\text{-DPP}(L_{t,1})) + k\log\left(|\mathcal{X}|\right) \right) + C^\prime \gamma_{TB}} \geq \gamma_{TB}$, therefore we gain nothing as opposed to just bounding with the maximum information gain $\gamma_{TB}$. For this reason, the bound is always looser than the original GP-UCB-PE bound.

\section{ADDITIONAL EXPERIMENTS}\label{appendix:addexp}

\begin{figure}[H]
    \includegraphics[width=\textwidth]{{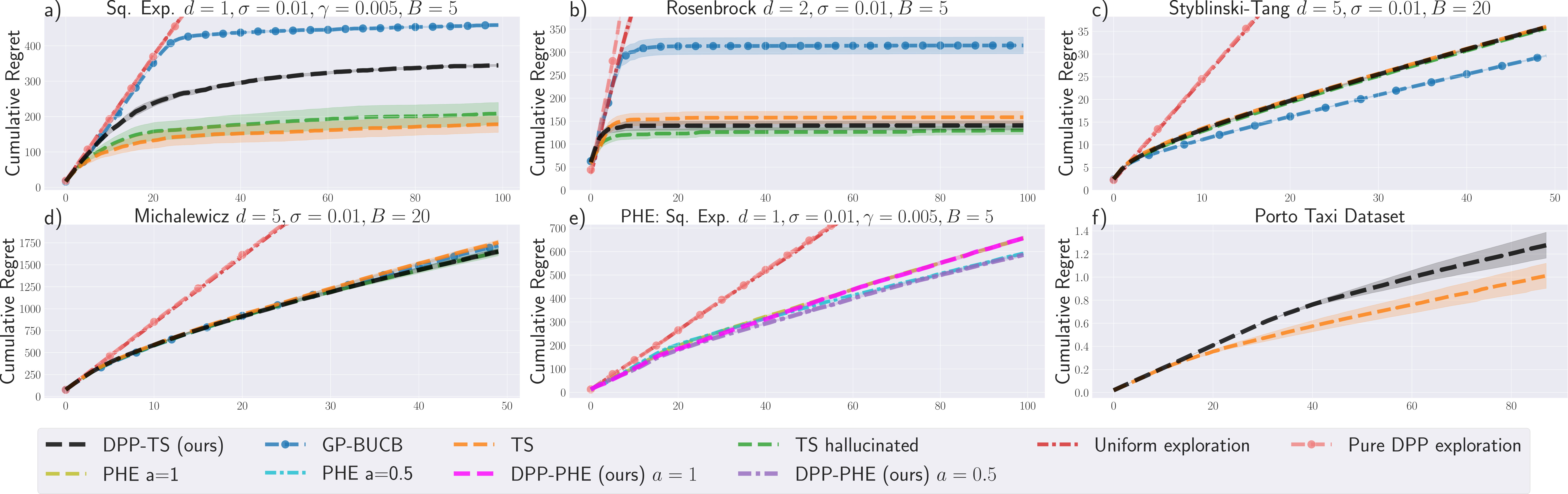}}
    \caption{Comprehensive experimental comparisons between DPP-TS and classic BBO techniques for Cumulative Regret, featuring the same experimental settings as those shown in Figure \ref{fig:experimentcomp}. DPP-TS is no longer the best performing algorithm according to this metric, as the DPP component favors additional and potentially suboptimal exploration through batch diversification. However, even when it is not better, DPP-TS still quickly converges to asymptotic behavior identical to that of TS. In Appendix \ref{lambdaexperiments} we illustrate a method for overcoming DPP-TS's limitations on Cumulative Regret.}
    \vspace{-0.3cm}
    \label{fig:experimentcomp-cumulative}
\end{figure}

\subsection{Cumulative Regret}

When performing the experiments from Section \ref{expsection} we mainly track Simple Regret, our target metric of choice on which we prove our theoretical bounds. Optimizing for Simple Regret corresponds to searching for good maximizers and heavily favors exploration, therefore we heavily benefit from DPP-BBO's batch diversification properties. However, we still wish to track the classic Cumulative Regret performance of our algorithms, in order to gain insight into whether DPP-BBO can be used to also optimize for such a metric.

Overall, when compared to classic TS and hallucinated TS on Cumulative Regret as seen in Figure \ref{fig:experimentcomp-cumulative}, DPP-TS is no longer the best performing algorithm, often over-exploring at the beginning, but still quickly converging to sublinearity. In other cases, its performance is virtually identical. For those situations in which the DPP component causes excessive exploration, we propose a solution (in Appendix \ref{lambdaexperiments}) involving limiting the use of DPP-TS to an \textit{initialization phase}.

\subsection{$\lambda$-parametrized DPP kernel} \label{lambdaexperiments}

In order to explicitly control the degree of exploration induced by the DPP reweighting of $P_{\max}$, we can parametrize our sampling distribution $P_{\text{DPP-TS}}$ with a $\lambda$ exploration parameter.

Using $\lambda \in [0,\infty]$, we would like a parametrization such that:
\begin{itemize}
\item For $\lambda = 1$ we recover the original formulation from Definition \ref{dpptsdef}: $P_{\text{DPP-TS}}(X) \propto P_{\max t}(X) \det(I + \sigma^{-2}{K_t}_X)$;
\item For $\lambda = 0$ we obtain regular Thompson Sampling $P_{\max}$;
\end{itemize}

A proposal for such a $P_{\text{DPP-TS}}$ parametrization is to use a multiplicative $\lambda$, such as
\begin{equation} \label{multlambda}
P_{\text{DPP-TS}}(X) \propto P_{\max t}(X) \det(I + \lambda\sigma^{-2}{K_t}_X)
\end{equation}

In Figure \ref{fig:experiment-lambdamult} we illustrate an experiment comparing TS and parametrized DPP-TS with different values for $\lambda$. It is straightforward to observe that by interpolating between TS and DPP-TS we observe a tradeoff in Simple Regret against Cumulative Regret performance. Smaller values of $\lambda$ correspond to slower convergence in Simple Regret but overall lower Cumulative Regret.

\begin{figure}[H]
    \includegraphics[width=\textwidth]{{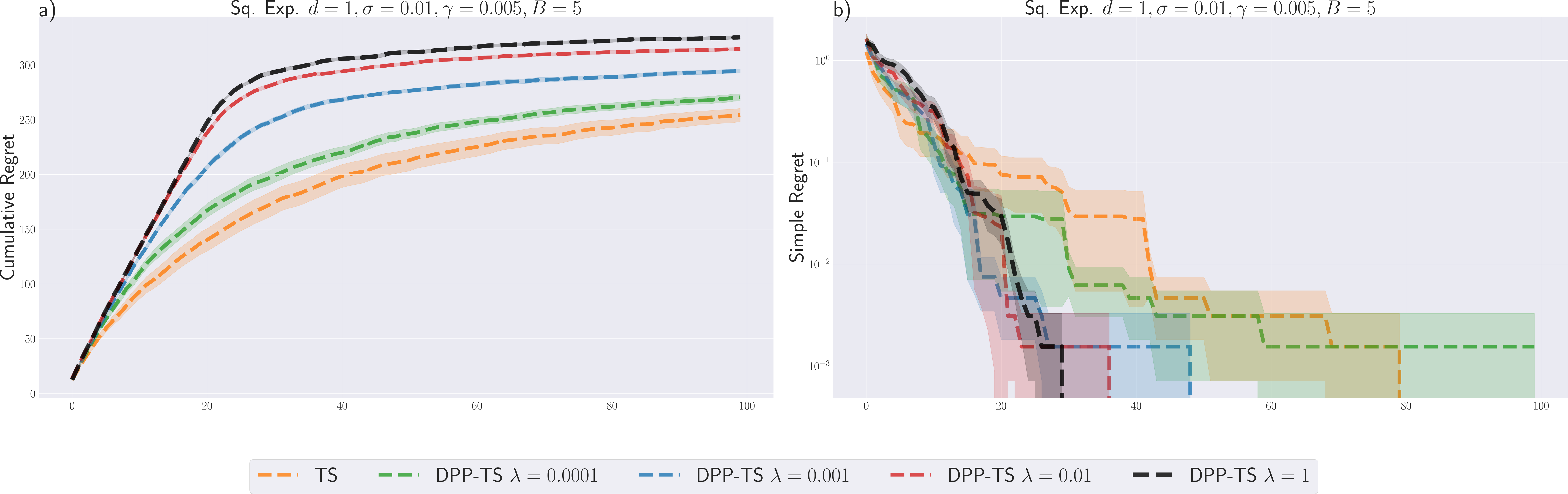}}
    \caption{Experimental comparison of DPP-TS with different $\lambda$ parameter in the parametrized mutual information DPP kernel, for both Cumulative and Simple Regret. Changing $\lambda$ corresponds to interpolating between the behaviors of TS and DPP-TS, with TS favoring Cumulative Regret and DPP-TS favoring Simple Regret.}
    \vspace{-0.3cm}
    \label{fig:experiment-lambdamult}
\end{figure}

Moreover, we investigate the use of time-varying $\lambda_t$ for the purpose of using DPP-TS as an \textit{initialization} phase for regular TS. Figure \ref{fig:experiment-lambdainit} illustrates a successful example of such a procedure: $\lambda_t$ is set to be equal to 1 (equivalent to the original DPP-TS formulation) up until iteration $T_{\text{init}}$, then equal to 0 (equivalent to TS). By doing this, we can constrain the DPP-TS over-exploration behavior to the very first batches we evaluate, and obtain both lower/equivalent Cumulative Regret than regular TS and lower Simple Regret, as seen in the cases for $T_{\text{init}} = 15$ and $T_{\text{init}}=24$.

The experiments for this Section all involve 1-d true functions sampled from a Gaussian Process with Squared Exponential kernel and $\gamma = 0.005$, evaluated on a discrete grid of 1024 points, with observational noise of $\sigma = 0.01$. The algorithms use a correctly-specified internal GP model with the same parameters of the true GP prior, and batch size is $B = 5$.

\begin{figure}[H]
    \includegraphics[width=\textwidth]{{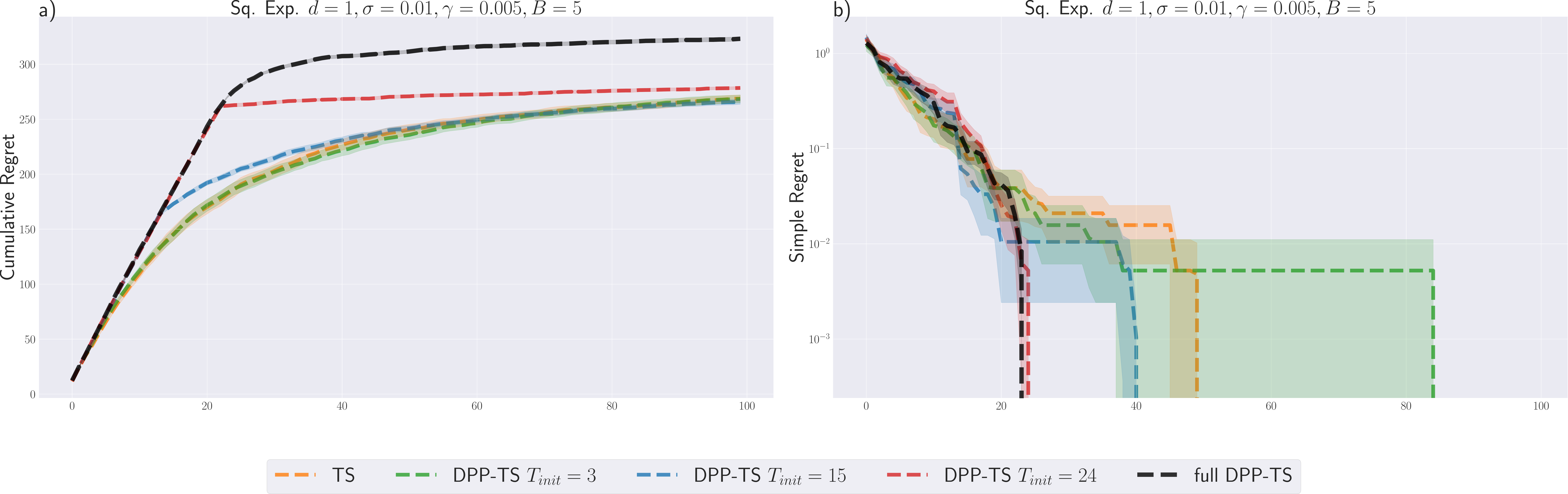}}
    \caption{Experimental comparison of DPP-TS when used as an initialization scheme for the first $T_{\text{init}}$ iterations, for both Cumulative and Simple Regret. It is apparent how, with properly tuned $T_{\text{init}}$, it is possible to maintain the fast Simple Regret convergence properties of DPP-TS while not overshooting classic TS in Cumulative Regret.}
    \vspace{-0.3cm}
    \label{fig:experiment-lambdainit}
\end{figure}

\subsection{DPP-TS and DPP-TS-alt comparison}

Throughout this work, we refer to DPP-TS as the procedure formalized in Algorithm \ref{dpptsalgo}, which is a simple and natural diversifying extension of classic TS. However, the algorithm we prove our theoretical Bayes Simple Regret bound on with Theorem \ref{dpptsbatchbound} is a slightly different procedure which we name DPP-TS-alt, that differs from DPP-TS in that it selects the first point of every batch with standard TS, a technicality required for the proof technique to work. As mentioned in the main text, the reason why we introduced DPP-TS as such and not DPP-TS-alt in the first place is both for simplicity and because in practice their performance is virtually identical.

In Figure \ref{fig:experiment-dpptsalt} we illustrate an experiment comparing DPP-TS and DPP-TS-alt (with classic TS for reference), showing performance of DPP-TS and DPP-TS-alt to be equivalent in both Cumulative and Simple Regret.

\begin{figure}[H]
    \includegraphics[width=\textwidth]{{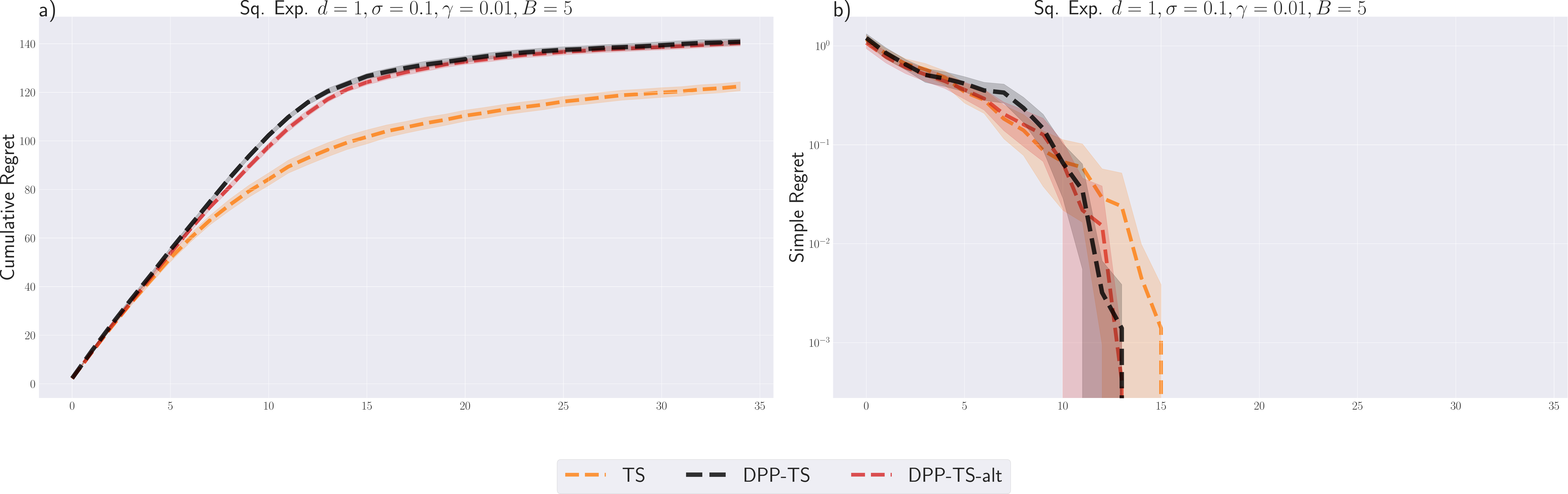}}
    \caption{Experimental comparison of DPP-TS and DPP-TS-alt, for both Cumulative and Simple Regret.}
    \vspace{-0.3cm}
    \label{fig:experiment-dpptsalt}
\end{figure}

\end{appendices}

\end{document}